\documentclass[preprint,11pt,authoryear]{elsarticle}
\usepackage{amsmath,amssymb,amsthm}
\usepackage{graphicx}
\usepackage{booktabs}
\usepackage{float}
\usepackage{longtable}
\usepackage{etoolbox}
\AtBeginEnvironment{table}{\renewcommand{\baselinestretch}{1.0}\selectfont}
\usepackage{array}
\usepackage[T1]{fontenc}
\usepackage{xr-hyper}
\usepackage{hyperref}
\hypersetup{colorlinks=true,linkcolor=[rgb]{0.1,0.25,0.55},citecolor=[rgb]{0.1,0.25,0.55},urlcolor=[rgb]{0.1,0.25,0.55}}
\newtheorem{proposition}{Proposition}
\newtheorem{assumption}{Assumption}
\newtheorem{lemma}{Lemma}
\newcommand{\pp}{\,pp}

\usepackage{amsmath,amsfonts,bm}

\def\eqref#1{equation~\ref{#1}}

\def\1{\bm{1}}

\DeclareMathAlphabet{\mathsfit}{\encodingdefault}{\sfdefault}{m}{sl}
\SetMathAlphabet{\mathsfit}{bold}{\encodingdefault}{\sfdefault}{bx}{n}

\usepackage[switch]{lineno}
\makeatletter
\def\ps@pprintTitle{\let\@oddhead\@empty\let\@evenhead\@empty\let\@oddfoot\@empty\let\@evenfoot\@empty}
\makeatother
\usepackage[section]{placeins}
\begin{document}

\begin{frontmatter}
\title{Separating Decision Time from Decision Quality in the Real-Time Gap of Distilled Deciders: Evidence from a Game and a Conveyor Simulator}
\author[a]{Chihoon Shin\corref{cor1}}
\ead{chihoon@mssj.ai}
\author[a]{Junyeong Lee}
\ead{junyeong@mssj.ai}
\author[b]{Kihyeok Jeong}
\ead{knp6464@korea.ac.kr}
\author[c]{Wonok Kwon}
\ead{happy@etri.re.kr}
\cortext[cor1]{Corresponding author.}
\affiliation[a]{organization={Myeongseongsimjae AX Institute},city={Yeongyang-gun},state={Gyeongsangbuk-do},country={Republic of Korea}}
\affiliation[b]{organization={Department of Computer Science and Software Engineering, Korea University},city={Seoul},country={Republic of Korea}}
\affiliation[c]{organization={Electronics and Telecommunications Research Institute (ETRI)},city={Daejeon},country={Republic of Korea}}
\begin{abstract}
Real-time agents are often evaluated with the world paused, or with decision time rounded to whole ticks (tick conversion). Tick conversion predicts almost no loss for a learned decider whose hit accuracy falls by 14.8 percentage points (pp) on the same games in asynchronous play. We instead split the decider's gap to a zero-latency teacher into time and quality components. In ViZDoom, we train a small decider by imitating a scripted teacher, run the game on wall-clock time, and add a control in which the teacher waits for a call to the decider's server before deciding. On a Windows host, three preregistered studies put the time component at 6.5--8.7\pp{} and the quality component at 6.3--8.6\pp. On a Linux host that skips 0.08--0.09\% of ticks, the time component disappeared ($-$0.1\pp; paired 95\% bootstrap interval [$-$0.4, 0.0]) while the quality component remained (5.6\pp{} [3.7, 7.5]). Retraining on teacher-labelled states from the decider's own play (DAgger) improved it on new games on both hosts (2.6\pp{} [0.4, 4.9] and 3.1\pp{} [1.1, 5.2]), a preregistered partial success. Imposing one delay schedule on every arm in live play left a quality component of 5.8\pp{} [4.0, 7.5] on 100 new games, and retraining cut the decider's disagreement with the teacher on its own states from 22.7\% to 14.9\%. In a conveyor simulator with a 400\,ms deadline, delay past it erased the quality component. A latency-matched control whose delays match the decider's shows which remedy to try.
\end{abstract}
\end{frontmatter}

\section{Introduction}
\label{sec:intro}

In real-time systems the world does not wait: a game engine keeps running, and a part on an inspection line reaches its diverter whether or not a decision has arrived. Language-model agents are being tried as controllers in such live loops, games among them \citep{park2025,golchinfar2026}. When such a decider loses skill in real time, two different things may have gone wrong: its decisions may arrive too late, or they may be worse decisions. The remedies differ. The first calls for faster inference, prediction or latency-aware policies \citep{firoiu2018,chen2021}; the second calls for better training. Common evaluation approaches do not separate them: paused evaluation removes time; tick conversion, which turns measured latency into held engine ticks, models time but not what the host does; and game studies report latency beside the score.

We ask: \emph{when a fast distilled decider loses skill in real-time play, how much of the loss is decision time and how much is decision quality, and can training reduce the quality component?} Because the answer turns out to depend on the host and the scene, we also ask which part of each component the host, the scene and the decider set. We use a first-person game as a controlled testbed, and a conveyor-inspection simulator with a hard deadline as a second environment.

Our answer rests on a latency-matched control: a teacher arm that decides with the teacher's rules but waits for a real call to the decider's server before every decision. It is intended to reproduce the decider's waiting time while retaining the teacher's decisions. Comparing this latency-matched teacher with the zero-latency teacher gives the time component. Comparing the latency-matched teacher with the decider gives the quality component; Figure~\ref{fig:teaser} shows both components in the three studies on host 1. Both components split the decider's gap to the zero-latency teacher; the 14.8\pp{} of the abstract is a different contrast, between two evaluations of the same decider (Section~\ref{sec:shortcuts}). This quality component also includes the effect of replacing exact coordinates with a textual description of the state; for hit accuracy we find no measurable input loss (Section~\ref{sec:hosts}).

\begin{figure}[!tb]
\centering
\includegraphics[width=\linewidth]{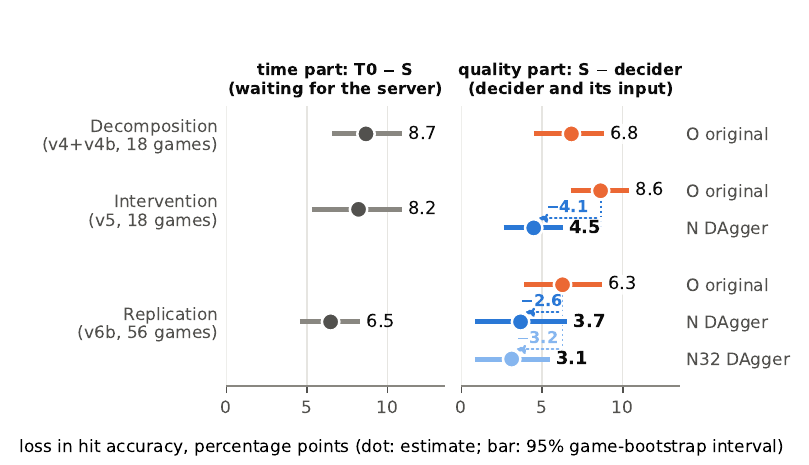}
\caption{Time and quality components in three studies on host 1 (v4+v4b, v5, v6b; 18, 18 and 56 distinct games). Left: time component T0 $-$ S. Right: quality component S $-$ decider for O, N and N32, including the input change and, in v5, imperfect latency matching; arrows: reduction from O. Bars: paired 95\% game-bootstrap intervals. Symbols: Table~\ref{tab:notation}.}
\label{fig:teaser}
\end{figure}

On the first host both components turn out to be large and of similar size. On a second host that almost never skips a tick, the time component disappears and the quality component stays (Section~\ref{sec:hosts}): in defend\_the\_center the time component tracked the host while the quality component stayed with the decider; in deadly\_corridor, where the scripted teacher is weaker, the quality component reversed. We interpret this disagreement as consistent with covariate shift, which imitation learning is known to suffer \citep{ross2011}. Because the quality component persists where almost no ticks are skipped, the shift is not produced by skipped ticks alone; we attribute the remainder to the decider's errors leading it into states absent from the teacher-played data, which we did not measure directly on the second host.

Collecting those states in the live loop and labelling them with the teacher reduces the quality component. The reduction replicates on 56 new distinct games for two training seeds on the first host and for one on the second, and on the first host it is not explained by data volume: adding the same number of teacher-played states instead did not give the same gain (Sections~\ref{sec:control} and~\ref{sec:replication}). The time component is untouched, as it must be, since no learning enters the teacher arms; reducing it would require changes to inference, the execution loop or the policy's handling of delay, which we do not attempt. Our contribution is therefore an evaluation method and what it reveals, not a method that shortens decisions. The method needs a teacher that can be run under the decider's timing, such as the scripted teachers used here; given one, it indicates whether to make the decider faster or to retrain it before either is attempted. We test it in imitation-learning evaluations with scripted teachers; whether the remedy it indicates is the better one for other decider families is untested.

Our contributions are as follows.
\begin{itemize}
\item A \textbf{latency-matched decomposition} of a real-time gap into a time and a quality component, with a check, for every teacher and decider run pair, of whether both lost a similar share of engine ticks to waiting (ratio within 0.8--1.25; pairs outside the band are flagged, not dropped), and an \textbf{analysis} in a model of asynchronous play: under a Lipschitz assumption with unknown constant the manipulation check becomes a bias bound, and, for additive return, the quality component is bounded by the decider's disagreement with the teacher on its own asynchronous states (Sections~\ref{sec:method} and~\ref{sec:theory}, proofs in Section~\ref{app:proofs}).
\item A preregistered \textbf{measurement} of both components for a distilled decider in asynchronous ViZDoom play, reported with game-level and run-level uncertainty and repeated on two hosts and two scenes, with a decision-period study, a live study that imposes one delay schedule on every arm, and a controlled-timing emulator, its fidelity confirmed on fresh seeds, that measures the share-matching bias directly (Sections~\ref{sec:decomp}, \ref{sec:live}, \ref{sec:hosts} and~\ref{sec:controlled}).
\item A \textbf{diagnosis} of the quality component as consistent with covariate shift, supported by the decider's measured disagreement with the teacher on its own states, and an \textbf{intervention}, on-policy asynchronous distillation, evaluated against the original decider in interleaved triads with a data-volume control and a replication on 56 new distinct games. The intervention met its preregistered criteria only in part: the remaining quality component stayed above 2\pp{} (Sections~\ref{sec:diag}--\ref{sec:replication}).
\item \textbf{Evaluation lessons}: tick conversion predicts almost no loss where the observed loss is 14.8\pp{}, concurrent evaluation on one host fails a preregistered calibration, and in a conveyor-inspection simulator a deadline turns the time component into a step whose edge the host's latency tail moves (Sections~\ref{sec:factory} and~\ref{sec:shortcuts}).
\end{itemize}

\section{Related work}
\label{sec:related}

\paragraph{Acting while the world moves} Real-time reinforcement learning starts from the observation that the state keeps changing while an action is selected \citep{ramstedt2019}; neither the Markov nor the semi-Markov decision process captures a state that changes during the agent's own computation \citep{travnik2018}. Agents held to human reaction times lose performance quickly \citep{firoiu2018}.

Delay-aware formulations treat delay as a known constant, as actions that execute a fixed number of steps later \citep{derman2021} or through state augmentation \citep{chen2021}. On physical systems the timing itself is uncertain and learning is sensitive to setup details that are rarely reported \citep{mahmood2018,dulacarnold2019}. Benchmarks on real robots therefore specify action and read-write cycle times. They also separate environment and agent computation into different processes to control delays \citep{mahmood2018benchmarking}. Sequential learning has been found to degrade relative to asynchronous learning when updates are expensive \citep{yuan2022async}.

Work on concurrent control, observation delay and language-model game agents is reviewed in Section~\ref{app:from-related}.

\paragraph{Fast-slow and distilled agents} Several systems couple a slow reasoner with a fast controller. A reasoning vision-language model that needs about 1.5\,s per response, too slow for a 15\,Hz control loop, is paired with a reactive model that answers in milliseconds \citep{li2026}. These works report scores and latencies side by side without attributing the loss between them. Our intervention uses dataset aggregation (DAgger) \citep{ross2011}: we collect states reached by the decider during asynchronous play and label them with the teacher's decisions for training. We use DAgger unchanged and apply it inside the asynchronous loop, where it can reduce the quality component but not the time component.

\paragraph{Paused, converted and real-time evaluation} Pausing the game during inference turns a real-time game into a turn-based one; vision-language models score higher on paused than on real-time versions of the same games, although both versions stay nearly unsolved \citep{zhang2025videogamebench}, reasoning latency limits decision updates to second-level intervals \citep{chen2024can}, and a real-time variant of a paused benchmark makes latency part of task performance \citep{ouyang2026gameworld}. Tick conversion, as we use the term, belongs to this family; we test it against asynchronous play rather than propose it. Closely related are frame skipping and action repetition: static frame skip repeats an action for a fixed number of frames \citep{srinivas2016dynamic}, action repetition requires open-loop action sequences but can shorten the effective horizon \citep{kalyanakrishnan2021an}, and learning the action duration adaptively did not beat a well-chosen fixed value in a fighting game \citep{nguyen2026for}. A skipped tick in asynchronous play is an action repetition that the agent did not choose.

\paragraph{Evaluation methodology} Shared protocols \citep{bellemare2013,machado2018}, interval estimates and seed counts \citep{henderson2018,agarwal2021,colas2018}, sources of benchmark variance \citep{bouthillier2021}, environment diversity \citep{cobbe2019}, preregistration \citep{hofman2023} and negative results \citep{karl2024} all bear on how agents should be compared. Prior work models delays, designs systems to tolerate latency and reports real-time performance. We measure how much of a real-time performance loss is associated with latency and how much with decision quality. We add three checks specific to real time: a latency-matched arm, distinct games as the resampling unit, and a calibration of whether a host can run evaluations concurrently.

\section{Setting}
\label{sec:setting}

\paragraph{Task and agents} We compare a scripted teacher that reads exact enemy coordinates with a learned decider that selects actions from short state sentences. The task uses aiming, firing and strafing actions in ViZDoom defend\_the\_center \citep{kempka2016} with two strafing buttons added, giving seven composite actions: turn left, turn right, fire, strafe left, strafe right, fire while strafing left, and fire while strafing right. Figure~\ref{fig:env} shows the game and the conveyor simulator used as a second environment (Section~\ref{sec:factory}).

\paragraph{Teacher} The teacher is a script that reads exact coordinates: it targets the nearest enemy, fires only when the enemy lies inside an aim cone of $\arctan(16/d)$ degrees clipped to 1--6 degrees and at least 14 ticks after its previous shot, and alternates strafing direction every 10 ticks.

\paragraph{Decider} The decider is the learned action-selection model. It never sees coordinates. A perception script writes a short state sentence (up to three enemies with side, 3-degree angle bin and 100-unit distance bin; enemy count; strafing phase; gun ready; ammunition in 5-bullet bins), which the decider encodes with hashed n-grams into 1,785 features. The decider has 265,665 parameters. It is trained by imitation for 2,500 steps on 300 teacher games from training seeds. On a CPU, it decides in about 2\,ms (p50 2.14\,ms, p90 15.94\,ms).

\paragraph{Environment} The engine advances 35 ticks per second, one every $\Delta = 28.6$\,ms; the pistol holds 26 bullets with no refill; an episode ends after 2100 ticks or at death.

Figure~\ref{fig:arch} shows the evaluation loop, the arms and the training pipeline, and Figure~\ref{fig:design} the arms of the latency-matched decomposition. Figure~\ref{fig:decider} shows the decider's network, and Section~\ref{app:loopsupp} the loop of the earlier tick-converted evaluations.

\begin{table}[!htbp]
\centering
\caption{Arms and symbols. Colours of these arms in the data figures: T0 dark gray, S gray, O orange, N dark blue, N32 light blue; other series use colours defined in their own figures.}
\label{tab:notation}
\small
\begin{tabular}{@{}lp{0.74\linewidth}@{}}
\toprule
Symbol & Meaning \\ \midrule
T0 & teacher with no added latency \\
S & the same teacher waiting for the decider's server before acting (latency-matched teacher) \\
SS & the teacher deciding from the state sentence the decider reads (v10) \\
D & a learned decider on the server, in Sections~\ref{sec:method} and~\ref{sec:theory} (O, N, N32 or M) \\
O & original imitation decider \\
N & decider after two on-policy DAgger rounds, in which the decider plays, the teacher labels the states it reaches, and the decider is retrained on them; N1 after one round \\
N32 & the same data as N, trained with seed 32 \\
M & data-volume control: the original recipe plus the same number of teacher-played rows \\
H$k$ & O holding each action for $k$ ticks after deciding (v15) \\
L8 & O acting every tick on the decision it made 8 ticks earlier (v15) \\
pp & percentage points \\
distinct games & seeds whose teacher plays identical games form one unit; all intervals resample these units \\
\bottomrule
\end{tabular}
\end{table}

\begin{figure}[!htbp]
\centering
\includegraphics[width=\linewidth]{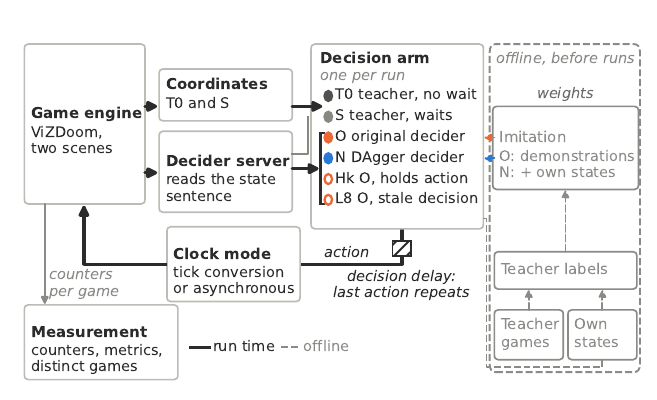}
\caption{The evaluation loop. The engine sends exact coordinates to the teacher arms and a state sentence to the decider's server. Tick conversion holds the previous action for $m(L)$ ticks; asynchronous play runs at 35 ticks per second and repeats the last action while a decision is pending (hatched bar). Dashed: offline training data. Symbols: Table~\ref{tab:notation}.}
\label{fig:arch}
\end{figure}

\paragraph{Two clocks} Under \emph{tick conversion}, a decision that takes $L$\,ms holds the previous action for $m(L)=\lceil L/\Delta\rceil-1$ ticks, a constant execution delay in the sense of \citet{derman2021}. Under \emph{asynchronous play}, the engine runs in ViZDoom's asynchronous player mode at 35\,Hz of wall-clock time and does not wait for the agent. A tick is \emph{skipped} when it passes without a new state being read. Skipped ticks are observed from the engine's episode clock; held ticks are modelled.

\paragraph{Metric and unit} The primary metric is hit accuracy, $\sum\mathrm{hits}/\sum\mathrm{shots}$, with shots counted from ammunition decreases and hits from the engine counter; before any learning we checked that it separates scripts known to differ (Section~\ref{app:ruler}). Seeds are not independent games here: 200 seeds produced 103 distinct teacher games, one of which accounted for 98 seeds (Section~\ref{app:seeds}). We group seeds by the signature of the tick-converted teacher's game and treat each distinct game as one unit. We report two intervals: a percentile bootstrap over distinct games (10,000 resamples, seed 0, paired across arms), which captures game-to-game variation, and a $t$ interval over runs ($n=5$, df $=4$), which captures run-to-run variation on the host.

\section{Method}
\label{sec:method}

\subsection{Latency-matched decomposition}

\paragraph{Three arms} Three arms play the same seeds asynchronously. \textbf{T0} is the teacher with no added latency. \textbf{S} is the same teacher, but before applying each decision it sends the request that the decider would receive to the decider's server and waits for, then discards, the answer. Each decision of S therefore carries the decider's transport and inference time, including its tail (the occasional much longer delays). \textbf{D} is the decider. With accuracy $a(\cdot)$ over distinct games,
\begin{equation}
a(\mathrm{T0}) - a(\mathrm{D}) = \underbrace{a(\mathrm{T0}) - a(\mathrm{S})}_{\text{time component}} + \underbrace{a(\mathrm{S}) - a(\mathrm{D})}_{\text{quality component}} .
\end{equation}
The step from T0 to S adds only the wait for the decider's server; the step from S to D replaces the teacher and its coordinate input with the decider and its sentence input (Figure~\ref{fig:design}).

\paragraph{When the decomposition is valid} The time component measures what waiting for the decider's server costs on this host, which includes transport and host jitter (variation in timing on the computer that runs the experiment) as well as inference. The decomposition is valid only if S and D experience similar delays. A tick is skipped when it passes without a new state being read, and an arm's skipped-tick share is the fraction of all ticks of the episode that were skipped. We check each pair with the ratio of skipped-tick shares $r = \mathrm{skip}(\mathrm{S})/\mathrm{skip}(\mathrm{D})$ against the band $0.8 \le r \le 1.25$ (the skip match), and flag rather than drop pairs outside it. S and D runs are interleaved on one host. In the two studies below, T0 comes from an earlier session on the same seeds (three runs, 0.1\% skipped ticks); the replication (Section~\ref{sec:replication}) measures T0 in the same session as the other arms.

\begin{figure}[!htbp]
\centering
\includegraphics[width=\linewidth]{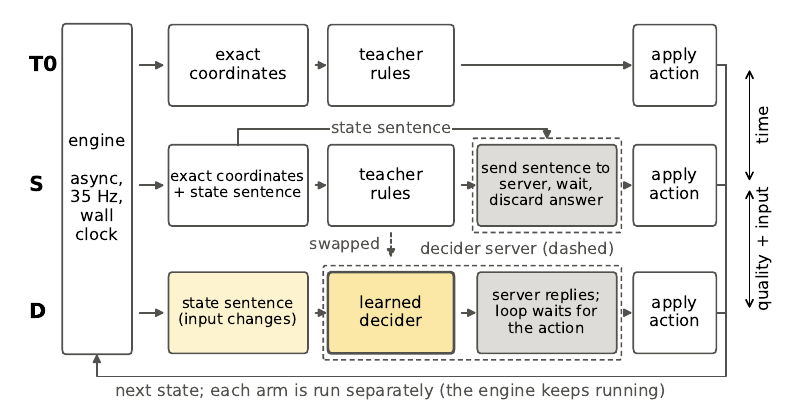}
\caption{The three arms of the latency-matched decomposition. S acts like T0 but calls the decider's server and waits for, then discards, the answer; D is the learned decider on that server. T0 to S adds the wait (gray: time component); S to D swaps in the decider and its input (highlighted: quality component). Box widths do not encode time.}
\label{fig:design}
\end{figure}

\subsection{On-policy asynchronous distillation}

\paragraph{Idea} If the quality component comes from states the decider reaches only because ticks pass without its decisions, it should shrink when those states enter training. We therefore use DAgger \citep{ross2011}: the decider plays, the teacher labels the states the decider reaches, and the decider is retrained on the teacher's demonstrations plus these labelled states. We run it with the learner in control ($\beta = 0$) inside the asynchronous loop:
\begin{enumerate}
\item Let the current decider $D_i$ play $n$ asynchronous games on collection seeds. At every decision, record the state sentence it saw; after its answer has been applied, record the teacher's action for the same observation. Labelling after the answer adds no time to the decision.
\item Train $D_{i+1}$ from scratch with the original recipe on the teacher's demonstrations plus all collected rounds, keeping each row in the same training, validation or calibration split.
\item Repeat for a number of rounds fixed before any evaluation, then evaluate the last decider once.
\end{enumerate}
We use $n=60$ games per round and two rounds, and fix the training seed in advance (seed 31, as for the original decider). A decider trained with seed 32 on the same data serves as a replicate.

\subsection{Protocol}
\label{sec:protocol}

\paragraph{Preregistration} Before its seeds were opened, every experiment was frozen in a preregistration with numeric verdict lines, and its text was hashed with SHA-256 (Section~\ref{app:prereg}). Figure~\ref{fig:ledger} lists every study with its host, scene, size and status; the full ledger is Table~\ref{tab:ledger}. Every analysis script was run on positive and negative controls before it saw the results. Seed ranges are disjoint across design, metric qualification, training, collection and evaluation. Doom evaluation runs are executed one at a time (Section~\ref{sec:shortcuts}). The conveyor study (Section~\ref{sec:factory}) ran four episodes concurrently, as preregistered; we report the consequences of that choice there.

\paragraph{Intervention verdicts} For the intervention we preregistered three verdicts. \emph{Failure}: the new decider N does not beat the original decider O, that is, the game-bootstrap lower bound of $a(\mathrm{N})-a(\mathrm{O})$ is at most zero. \emph{Strong success}: N beats O, the remaining quality component $a(\mathrm{S})-a(\mathrm{N})$ is at most 2\pp{}, and its upper bound lies below the lower bound of the time component. \emph{Partial success}: every other outcome.

How AI agents were used in the research, and the checks that kept their output from entering the results unchecked, are described in Section~\ref{app:from-ai}.
\section{Theoretical analysis}
\label{sec:theory}

\paragraph{Questions and scope} We examine when the decomposition separates waiting effects from policy effects, and how disagreement on states visited by the decider relates to the quality component. The disagreement result concerns additive return such as expected hits; its connection to hit accuracy is empirical, not part of the proposition. Proofs are in Section~\ref{app:proofs}.

\paragraph{Model} Let $s_t$ be the engine state at tick $t \in \{0,\dots,H-1\}$ and let $z_t \in \{0,1\}$ indicate whether the agent submits a new decision at tick $t$ ($z_t=0$ is a skipped tick, on which the previous action $a_{t-1}$ is repeated). The indicators are generated by a latent host state $h_t$ that evolves as a Markov chain; we call its law the \emph{skip process} $Z$. The augmented state $x_t = (s_t, a_{t-1}, z_t, h_t)$ is Markov. A policy $\pi$ maps $x_t$ to an action (in practice it ignores $h_t$), and is constrained to $\pi(x_t)=a_{t-1}$ whenever $z_t = 0$; the teacher reads $s_t$ exactly and the decider reads a sentence computed from $s_t$, so both are policies in this sense. Let $r(s_t,a_t)\in[0,1]$ be a per-tick reward (for example, a hit) and $J(\pi, Z) = \mathbb{E}\big[\sum_t r(s_t,a_t)\big]$ the expected return when $\pi$ plays under $Z$.

\begin{assumption}[Exogenous skipping]
\label{ass:exo}
The skip process does not depend on which action the policy chooses: the host state $h_{t+1}$, and hence $z_{t+1}$, is independent of $(a_0,\dots,a_t)$ given $h_t$ and the decision latency.
\end{assumption}

Under Assumption~\ref{ass:exo}, S and D share one skip process $Z_D$ when their latencies match, T0 plays under $Z_0$, and the three arms estimate $J(\pi_T, Z_0)$, $J(\pi_T, Z_D)$ and $J(\pi_D, Z_D)$.

\begin{proposition}[Identification and bias]
\label{prop:bias}
The time component $J(\pi_T,Z_0)-J(\pi_T,Z_D)$ changes only the skip process and the quality component $J(\pi_T,Z_D)-J(\pi_D,Z_D)$ changes only the policy. If instead S runs under $Z_S \neq Z_D$, the estimated quality component is $Q + [J(\pi_T,Z_S) - J(\pi_T,Z_D)]$, where $Q = J(\pi_T,Z_D)-J(\pi_D,Z_D)$ is the quality component under $Z_D$. If the teacher's value of a metric (return, or accuracy) is $\kappa$-Lipschitz in the expected share of skipped ticks $\sigma(Z)$ over the skip processes the host produces, the bias is at most $\kappa\,|\sigma(Z_S)-\sigma(Z_D)|$.
\end{proposition}

Section~\ref{app:from-theory} discusses how the manipulation check becomes a sensitivity estimate, why disagreements on skipped ticks cost nothing, and how the bound relates to DAgger and to hit accuracy.

\begin{proposition}[Quality component and on-policy disagreement]
\label{prop:pdl}
Let $d^{D}_t$ be the distribution of $x_t$ when the decider plays under $Z$. Let $\epsilon_D = \frac{1}{H}\sum_t \Pr_{x\sim d^{D}_t}\big[z_t=1,\ \pi_D(x)\neq\pi_T(x)\big]$ be its per-tick rate of disagreement with the teacher on its own states. Let $\Delta_Q = \max_{t,x,a}\big[Q^{T}_t(x,\pi_T(x)) - Q^T_t(x,a)\big]$, where $Q^T_t(x,a)$ is the teacher's action value (expected return for a specified action followed by the teacher) under $Z$. Then
\begin{equation}
J(\pi_T,Z)-J(\pi_D,Z) = \sum_{t} \mathbb{E}_{x\sim d^D_t}\big[Q^T_t(x,\pi_T(x)) - Q^T_t(x,\pi_D(x))\big] \le H\,\epsilon_D\,\Delta_Q ,
\end{equation}
and the difference is non-negative if $\pi_T$ is greedy (choosing an action with the highest action value) with respect to $Q^T$.
\end{proposition}

\section{Results}
\label{sec:results}

Figure~\ref{fig:forest} collects the primary contrasts of all studies on one axis. In every asynchronous study of defend\_the\_center, pooled hit accuracy is highest for the zero-latency teacher (T0), lower for the latency-matched teacher (S, the same teacher waiting for the decider's server) and lowest for the deciders, although individual runs overlap (Figure~\ref{fig:sm_studies}). In both studies that include it, the decider retrained by DAgger (dataset aggregation: the decider plays, the teacher labels the states it reaches, and the decider is retrained on them) lies between S and the original decider.

\begin{figure}[!tb]
\centering
\includegraphics[width=\linewidth]{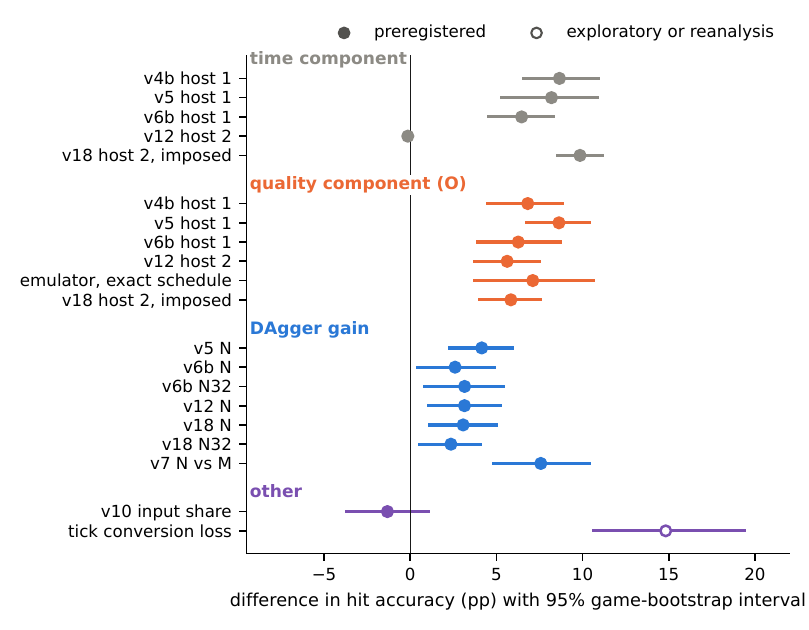}
\caption{Primary contrasts of all studies on one axis: time components, quality components of O, DAgger gains and two other contrasts, in hit-accuracy points with 95\% game-bootstrap intervals. Filled: preregistered; open: reanalysis (tick conversion over distinct games). Host 2: v12 and v18; all others host 1.}
\label{fig:forest}
\end{figure}

Four further metrics of the replication study (kills, reaction to a lone enemy, survival and hits taken) are shown in Figure~\ref{fig:sm_metrics} and discussed in Section~\ref{app:from-runvar}.

\begin{table}[!htbp]
\centering
\caption{Latency-matched decomposition of hit accuracy}
\label{tab:decomp}
\small
\begin{tabular}{@{}lccc@{}}
\toprule
Arm or contrast & Hit accuracy (\%) & \multicolumn{2}{c}{95\% interval} \\
 & or difference (pp) & game bootstrap & run-level $t$ \\ \midrule
T0 & 95.5\% & -- & -- \\
S & 86.8\% & -- & -- \\
O & 80.0\% & -- & -- \\
\addlinespace
Time, T0 $-$ S & 8.7\pp & [6.6, 10.9] & -- \\
Quality, S $-$ O & 6.8\pp & [4.5, 8.8] & [1.2, 12.4] \\
\bottomrule

\end{tabular}
\par\smallskip{\footnotesize\raggedright Five S/O pairs, 18 distinct games; T0 from an earlier session on the same seeds. Differences are computed from unrounded values. Game bootstrap: paired resampling of distinct games; run-level $t$: over the five pairs. Symbols: Table~\ref{tab:notation}.\par}
\end{table}

\subsection{Decision time and decision quality contribute comparably}
\label{sec:decomp}

Five interleaved S/O pairs on seeds 850000--850029 (30 seeds, 18 distinct games) give the decomposition in Table~\ref{tab:decomp}. The zero-latency teacher hits 95.5\% of its shots, the latency-matched teacher 86.8\% and the decider 80.0\%. The time component (T0 $-$ S) is therefore 8.7 percentage points (pp) and the quality component (S $-$ O) 6.8\pp. Counted in hits per game, an additive quantity, the same arms give a quality component larger than the time component (Figure~\ref{fig:hits}; exploratory).

The preregistered extension to five pairs allowed the reading ``quality component below 5\pp'' only if the game-bootstrap upper bound was at most 5\pp; it is 8.8\pp, so we do not use that reading.

The skip ratio fell outside the prespecified band in two of five pairs, with values of 1.59 and 0.62. The decider also varied more across runs than the latency-matched teacher (standard deviation 3.6 against 1.4\pp).

\subsection{Identical delays in live play}
\label{sec:live}

Proposition~1 requires S and the decider to face the same skip process; the studies above matched only the share of skipped ticks. In a further preregistered study (v18) we imposed one delay schedule on every arm in live asynchronous play. On host 2, which almost never skips ticks on its own, each arm waited after every decision until a preset number of ticks had passed: the observation gaps that the original decider O produced on host 1 in v9 (seeds 860000--860049). The teacher (Tw), O, N and a second DAgger decider trained with another seed (N32) received the same schedule on 100 new seeds (887000--887099, 100 distinct games); T0 played without waiting. The preregistered checks passed: the realised gap equalled the target gap in 99.9\% of decisions in every arm, and the waiting arms skipped 8.73--8.81\% of ticks (Section~\ref{app:live}).

\begin{figure}[!htbp]
\centering
\includegraphics[width=\linewidth]{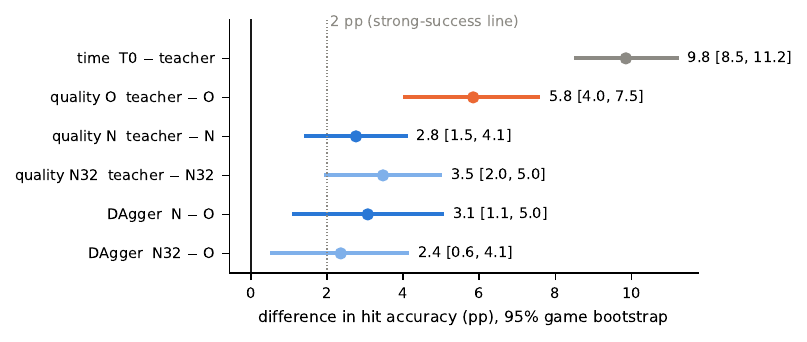}
\caption{Live play with one delay schedule imposed on every arm (v18; host 2; 100 distinct games, two sessions). Time: T0 $-$ teacher on the schedule; quality: teacher $-$ decider on the same schedule; DAgger: N or N32 $-$ O. Paired 95\% game-bootstrap intervals; dotted: 2\pp{} strong-success line.}
\label{fig:live}
\end{figure}

With delays identical by construction, the quality component of the original decider was 5.8\pp{} [4.0, 7.5] (Figure~\ref{fig:live}). The time component, T0 $-$ Tw, was 9.8\pp{} [8.5, 11.2]; on this host it is the cost of the imposed schedule, not of the host's own delays. DAgger improved on O for both training seeds, by 3.1\pp{} [1.1, 5.0] for N and 2.4\pp{} [0.6, 4.1] for N32. The quality component left for N was 2.8\pp{} [1.5, 4.1], above the 2\pp{} of the preregistered strong-success criterion, so the intervention stays a partial success. The estimate for O is slightly below the range of the share-matched studies on host 1 (6.3--8.6\pp) and within the interval of the synchronous emulator (7.1\pp{} [3.7, 10.6], Section~\ref{sec:controlled}); for N the emulator's estimate on the recorded host-1 schedules was lower ($-$0.2\pp{} [$-$2.5, 2.1], 29 games). These were measured on other games.

We ran the study in two sessions, restarting the servers and the runner in between. With the schedule imposed on this host, 498 of 500 seed and arm pairs produced the same episode in both sessions. The second session therefore shows that live play under an imposed schedule is reproducible; it does not add independent variation, and all intervals are over games.

\subsection{The quality component is consistent with covariate shift}
\label{sec:diag}

When ticks pass without a decision, the agent keeps its previous action and drifts into positions and timings the teacher never produced. We measured how often the decider chooses the teacher's action on two distributions of states (Table~\ref{tab:diag}). The decider agrees less often with the teacher on states encountered during its own play than on the teacher's states. On the teacher's own states the decider agrees 98.4\% of the time. On the decisions it took itself in 60 asynchronous collection games it agrees 88.8\% of the time, an error rate about seven times higher.

Between 44\% and 51\% of decisions with an enemy in view are taken after the ammunition has run out; there firing has no effect, and the teacher, which ignores ammunition, keeps labelling ``fire'', so we report agreement on decisions with ammunition left.

We read the gap as consistent with covariate shift: the decider encounters a different distribution of states during its own play than during teacher-generated training \citep{ross2011}. A direct count points the same way (exploratory): of the state sentences with an enemy in view and ammunition left that the decider met in the two DAgger rounds, 49.2\% and 47.8\% do not occur verbatim among the 40,440 distinct sentences of the teacher data it was trained on, against 8.0\% of the teacher's own validation sentences. Figure~\ref{fig:ood} shows how familiar these states are to the decider. Two competing explanations remain. One is data volume: the decider trained on the collected states also saw more data. The other is underfitting: the first retrained decider matched only 74.0\% of its own round-1 training rows when ammunition-empty rows are included. A data-volume control is reported in Section~\ref{sec:control}.

\begin{figure}[!htbp]
\centering
\includegraphics[width=\linewidth]{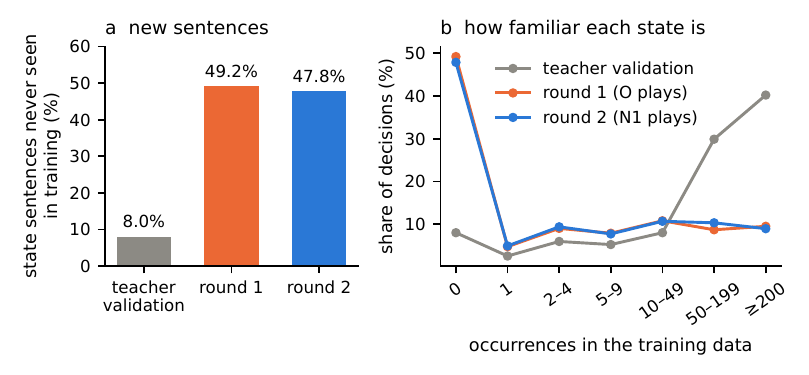}
\caption{A direct view of covariate shift (exploratory). (a) Share of state sentences, with an enemy in view and ammunition left, that never occur among the 40,440 distinct training sentences: 8.0\% of the teacher's validation sentences, 49.2\% and 47.8\% of those met in the two DAgger rounds. (b) How often each state occurs in the training data.}
\label{fig:ood}
\end{figure}

Most on-policy errors are turns in the wrong direction (Section~\ref{app:dagger}); Figure~\ref{fig:qual} shows six of them, one per game. Figure~\ref{fig:confusion} shows where the on-policy errors are: when the teacher turns right, the original decider turns left in about a fifth of the cases; strafing and firing decisions are almost always right.

Results from a second host qualify this mechanism: there almost no ticks pass without a decision and the quality component is of the same size (Section~\ref{sec:hosts}), so the drift is not produced by skipped ticks alone. The decider's own wrong turns may also carry it into states absent from the teacher's data, the setting in which DAgger was introduced \citep{ross2011}.

The live study also measured the mechanism directly. At every decision of a decider we recorded the action the teacher would have chosen in the same state. On its own states O chose differently from the teacher in 22.7\% of its decisions and N in 14.9\%, a difference of 7.8\pp{} [6.0, 9.6]; N32 differed in 13.8\%. Retraining on the decider's own states thus reduced disagreement on the states it visits, the quantity that bounds the quality component in Proposition~\ref{prop:pdl}, although disagreement remained. Across games the relation between a decider's disagreement and its quality loss was weak (Figure~\ref{fig:livemech}; correlation 0.24 for O, 0.19 for N and 0.06 for N32).

\begin{figure}[!htbp]
\centering
\includegraphics[width=0.62\linewidth]{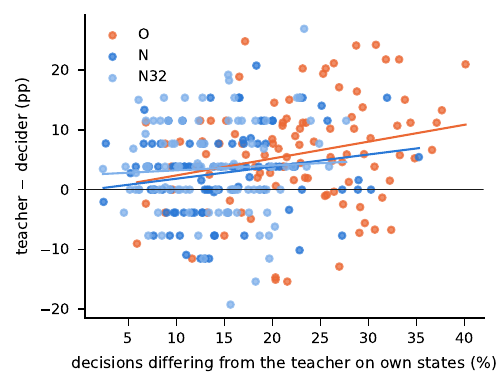}
\caption{Disagreement with the teacher on the decider's own states against its quality loss, per game (v18; 100 distinct games). Lines: least-squares fit per decider.}
\label{fig:livemech}
\end{figure}

\subsection{On-policy asynchronous distillation reduces the quality component}
\label{sec:dagger}

Retraining on decisions collected during the decider's own play reduced the quality component, with partial success under the preregistered criteria. Two rounds added 44,820 and 47,905 decisions collected in the loop. The final decider N was evaluated against the original decider O and the latency-matched teacher S in five interleaved triads on the same 18 distinct games (Table~\ref{tab:dagger}, Figure~\ref{fig:sm_studies}). On the preregistered game bootstrap, N improves on O by 4.1\pp{} (2.3--5.9), the quality component falls from 8.6\pp{} for O to 4.5\pp{} for N, and the time component is 8.2\pp. The remaining quality component exceeds 2\pp{} and its upper bound (6.3) lies above the lower bound of the time component (5.3), so the preregistered verdict is \emph{partial success}.

\begin{table}[!htbp]
\centering
\caption{Hit accuracy in the intervention study}
\label{tab:dagger}
\small
\begin{tabular}{@{}lccc@{}}
\toprule
Arm or contrast & Hit accuracy (\%) & \multicolumn{2}{c}{95\% interval} \\
 & or difference (pp) & game bootstrap & run-level $t$ \\ \midrule
T0 & 95.5\% & -- & -- \\
S & 87.3\% & -- & -- \\
N & 82.8\% & -- & -- \\
O & 78.7\% & -- & -- \\
\addlinespace
\textbf{Improvement, N $-$ O} & 4.1\pp & [2.3, 5.9] & [$-$0.4, 8.8] \\
Quality, S $-$ O & 8.6\pp & [6.8, 10.4] & [1.4, 16.0] \\
Quality, S $-$ N & 4.5\pp & [2.6, 6.3] & [1.4, 7.6] \\
Time, T0 $-$ S & 8.2\pp & [5.3, 10.9] & -- \\
\bottomrule

\end{tabular}
\par\smallskip{\footnotesize\raggedright Five interleaved triads (S, N, O), 18 distinct games; T0 from an earlier session. The last column is the preregistered run-level $t$ interval over the five triads. Symbols: Table~\ref{tab:notation}.\par}
\end{table}

The run-level interval of the improvement includes zero; run-level variation between triads is reported in Section~\ref{app:from-runvar}.

\subsection{Data-volume control}
\label{sec:control}

The improvement of N could come from the states it was trained on or simply from having seen more data. We therefore trained a control decider M using the same training procedure as N. M received the teacher's demonstrations plus 92,725 additional rows, exactly as many as the two DAgger rounds added. These additional rows came from new teacher-played games (off-policy). We evaluated M and N alternately on the replication seeds (preregistration v7).

Hit accuracy is 84.8\% for N (84.2\% in the replication session) and 77.2\% for M (Figure~\ref{fig:volume}), a difference of 7.6\pp. By the preregistered rule the improvement is \emph{not explained by data volume}. M also scored below the original decider measured earlier on the same seeds (81.6\%, Table~\ref{tab:replication}), although the two were not run in the same session: more teacher demonstrations did not help, and may have hurt, while the same number of on-policy rows helped. More training data alone therefore does not explain the gain; whether the first decider was also undertrained is not excluded (Section~\ref{sec:limits}).

\subsection{Replication on 56 new distinct games}
\label{sec:replication}

A second preregistered evaluation on 100 new seeds (860000--860099, 56 distinct games) addressed three objections to the first two studies: the small number of games, the earlier session of T0, and the single training seed. All five arms ran in the same session, one run at a time, in a rotated order over four blocks of 25 seeds. N32 was trained with training seed 32 on exactly the same data as N and was fixed before evaluation.

The replication again met the preregistered criteria for partial success (Table~\ref{tab:replication}). The contemporaneous T0 (94.3\%) is 1.2\pp{} below the earlier-session T0 of the first two studies, so their time components may be slightly overstated. Both DAgger deciders improve on O: by 2.6\pp{} for N (0.4--4.9; in the four blocks 5.0, 1.7, 1.7 and 2.4\pp) and 3.2\pp{} for N32 (0.8--5.4), leaving quality components of 3.7\pp{} and 3.1\pp. By the preregistered rules both are again a \emph{partial success}: the improvement's game-bootstrap interval excludes zero, and the remaining quality component is below the time component in point estimate, but its interval overlaps that of the time component and it is not below 2\pp.

\begin{table}[!htbp]
\centering
\caption{Replication with a second training seed}
\label{tab:replication}
\small
\begin{tabular}{@{}lccc@{}}
\toprule
Arm or contrast & Hit accuracy (\%) & Game-bootstrap & Skipped \\
 & or difference (pp) & 95\% interval & ticks (\%) \\ \midrule
T0 & 94.3\% & -- & 0.1\% \\
S & 87.9\% & -- & 8.5\% \\
N (seed 31) & 84.2\% & -- & 7.8\% \\
N32 (seed 32) & 84.7\% & -- & 8.2\% \\
O & 81.6\% & -- & 7.3\% \\
\addlinespace
Time, T0 $-$ S & 6.5\pp & [4.6, 8.3] & \\
Quality, S $-$ O & 6.3\pp & [3.9, 8.7] & \\
Quality, S $-$ N & 3.7\pp & [0.8, 6.6] & \\
\textbf{Improvement, N $-$ O} & 2.6\pp & [0.4, 4.9] & \\
Quality, S $-$ N32 & 3.1\pp & [0.8, 5.5] & \\
\textbf{Improvement, N32 $-$ O} & 3.2\pp & [0.8, 5.4] & \\
\bottomrule

\end{tabular}
\par\smallskip{\footnotesize\raggedright 56 distinct games with a contemporaneous zero-latency teacher and a second training seed (preregistration v6b); paired game-bootstrap 95\% intervals. Differences are computed from unrounded values. Skipped ticks: share of engine ticks that passed without a new decision. Symbols: Table~\ref{tab:notation}.\par}
\end{table}

\subsection{A second host, a second scene and the decision period}
\label{sec:hosts}

Four further preregistered studies ask which parts of the decomposition belong to the host, which to the scene and which to the decider, and how often the decider should decide (Table~\ref{tab:hosts}). Host 2 is a Linux workstation (Section~\ref{app:env}); the procedures, deciders, seeds and bootstrap are those of the earlier studies, and arms ran one at a time.

\begin{table}[!htbp]
\centering
\caption{Host, scene and decision-period contrasts}
\label{tab:hosts}
\small
\begin{tabular}{lcc}
\toprule
Study and contrast & Estimate (pp or points) & 95\% interval \\ \midrule
\multicolumn{3}{l}{\textit{v12, host 2, defend\_the\_center (hit accuracy, pp)}} \\
time T0 $-$ S & $-$0.1 & [$-$0.4, 0.0] \\
quality S $-$ O & 5.6 & [3.7, 7.5] \\
DAgger N $-$ O & 3.1 & [1.1, 5.2] \\
\addlinespace
\multicolumn{3}{l}{\textit{v11, host 2, deadly\_corridor (mean reward, points)}} \\
time T0 $-$ S & 0.0 & [0.0, 0.0] \\
quality S $-$ O & $-$4.6 & [$-$10.2, $-$0.1] \\
DAgger N $-$ O & $-$1.8 & [$-$7.5, 3.5] \\
\addlinespace
\multicolumn{3}{l}{\textit{v11, host 1, deadly\_corridor (mean reward, points)}} \\
time T0 $-$ S & $-$12.1 & [$-$25.7, 0.6] \\
quality S $-$ O & 2.5 & [$-$17.0, 21.0] \\
DAgger N $-$ O & 4.2 & [$-$15.7, 23.9] \\
\addlinespace
\multicolumn{3}{l}{\textit{v10, host 1, defend\_the\_center (hit accuracy, pp)}} \\
input S $-$ SS & $-$1.3 & [$-$3.7, 1.1] \\
policy SS $-$ N & 3.2 & [0.7, 5.7] \\
\addlinespace
\multicolumn{3}{l}{\textit{v15, host 2, defend\_the\_center (hit accuracy, pp)}} \\
hold 2 ticks, H2 $-$ H0 & $-$58.9 & [$-$67.6, $-$51.4] \\
fresh $-$ stale, H8 $-$ L8 & $-$7.8 & [$-$17.4, 1.2] \\
\addlinespace
\multicolumn{3}{l}{\textit{v15, host 2, deadly\_corridor (mean reward, points)}} \\
hold 1 tick, H1 $-$ H0 & $-$21.9 & [$-$35.6, $-$8.6] \\
hold 4 ticks, H4 $-$ H0 & 46.4 & [28.8, 64.2] \\
fresh $-$ stale, H8 $-$ L8 & 35.1 & [19.3, 50.7] \\
\bottomrule
\end{tabular}

\par\smallskip{\footnotesize\raggedright Preregistrations v10, v11, v12 and v15; paired game-bootstrap 95\% intervals over distinct games. Hit accuracy in pp for defend\_the\_center; mean reward for deadly\_corridor. Hold contrasts are written as H$k$ $-$ H0. Differences are computed from unrounded values. Symbols: Table~\ref{tab:notation}.\par}
\end{table}

\paragraph{Host (v12)} On host 2 we repeated the replication of Section~\ref{sec:replication} on the same 100 seeds (56 distinct games) with T0, S, N and O. The arms skipped 0.08--0.09\% of ticks instead of 7.3--8.5\%. The time component is $-$0.1\pp{} [$-$0.4, 0.0], below the preregistered 2\pp{} threshold, while the quality component of O is 5.6\pp{} [3.7, 7.5] (6.3\pp{} on host 1) and N improves on O by 3.1\pp{} [1.1, 5.2], again a partial success by the rule of Section~\ref{sec:dagger}. The decider's per-episode median decision time was also steadier on host 2, 1.8--2.3\,ms over 860 episodes, than on host 1, 2.4--21.2\,ms over 1,300 episodes (exploratory, from the run logs), so the time component measured on host 1 includes that host's scheduling. Figure~\ref{fig:latency_hosts} shows the two distributions.

\begin{figure}[!htbp]
\centering
\includegraphics[width=\linewidth]{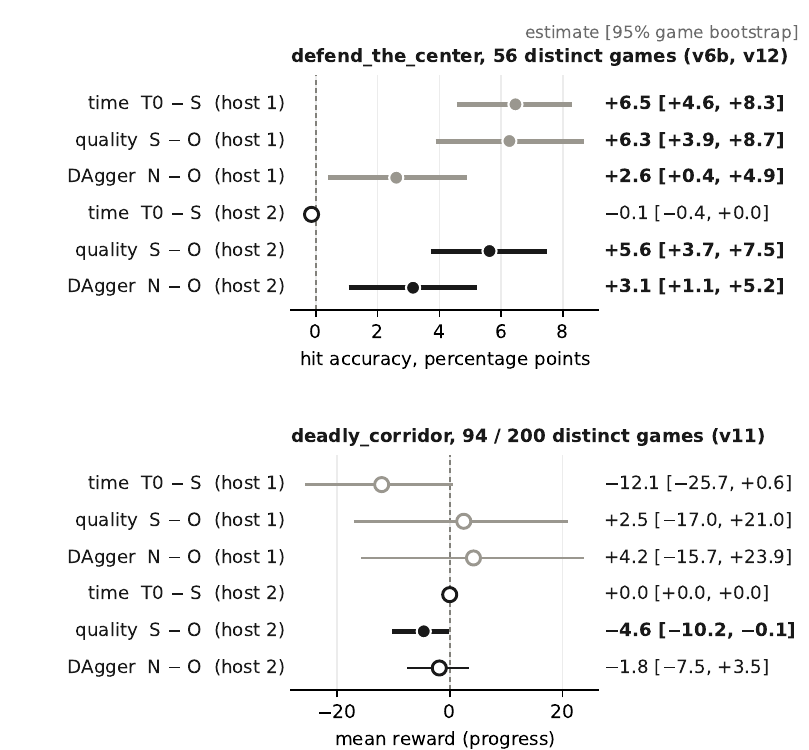}
\caption{Time (T0 $-$ S), quality (S $-$ O) and DAgger (N $-$ O) contrasts with 95\% game-bootstrap intervals; filled: interval excludes zero; gray: host 1, black: host 2. Top: defend\_the\_center, hit accuracy, 56 distinct games per host (v6b, v12). Bottom: deadly\_corridor, mean reward, 94 games on host 1 and 200 on host 2 (v11).}
\label{fig:hosts}
\end{figure}

\paragraph{Scene (v11)} deadly\_corridor rewards progress towards an armour at the end of a corridor; its teacher runs forward, sidesteps near enemies and never shoots, a script we know to be suboptimal. With 200 seeds (200 distinct games) on host 2, skipping stayed near 1.5\% and T0 and S played identical games, so the time component is exactly zero; the decider scored 4.6 points more than the teacher ([0.1, 10.2]) and DAgger did not help (N $-$ O $=$ $-$1.8 [$-$7.5, 3.5]), which is expected when the teacher is the weaker player. Whether waiting costs or pays is a property of the scene; how much waiting occurs is a property of the host (Figure~\ref{fig:hosts}). On host 1 the same 200 seeds gave only 94 distinct games: 107 of them produced one and the same game, in tick conversion and in asynchronous play alike, whereas all 200 differed on host 2. The mapping from seed to game thus depends on the host's installation; distinct games are identified within each host, so the two hosts' corridor results are not measured on the same games.

For hit accuracy, replacing exact coordinates with a sentence cost no measurable accuracy (v10, Section~\ref{app:from-v10}).

How often to decide differs by scene (v15, Section~\ref{app:from-v15}): deciding every tick is best in defend\_the\_center, holding each action for four ticks in deadly\_corridor.

An exploratory selective-automation analysis is in Section~\ref{app:from-cascade}.

\subsection{A second environment with a deadline}
\label{sec:factory}

In both Doom scenes a late decision leaves the previous action running, so waiting changes the score by degrees. On a production line a late decision can be worth nothing: a part reaches the diverter at a fixed time and, without a decision, a default applies. We therefore ran the decomposition, under a new preregistration, in a conveyor-inspection simulator in which each part reaches the diverter 400\,ms after the camera and passes by default (Section~\ref{app:factory}). A decider of the same family as the Doom decider reads a sentence describing each part and answers over HTTP; an exact teacher decides from the true state. We added a fixed delay after each answer and crossed it with the policy.

\begin{figure}[!htbp]
\centering
\includegraphics[width=\linewidth]{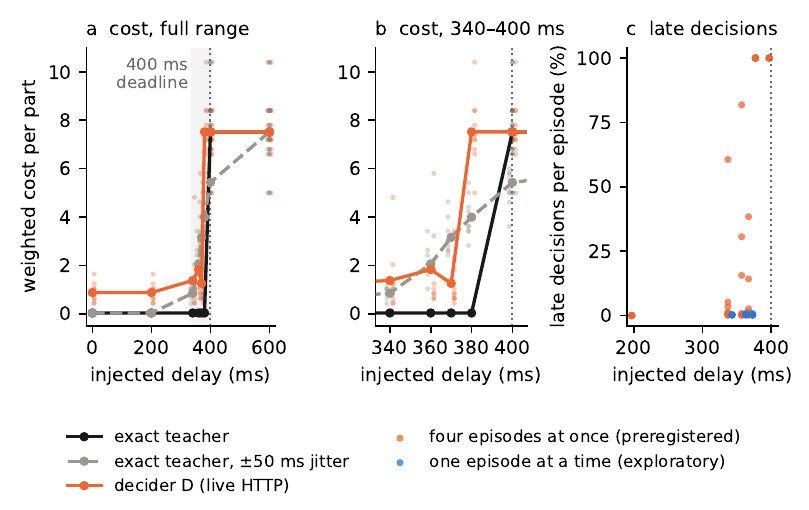}
\caption{A deadline makes the time component a cliff (conveyor simulator, line A, asynchronous). (a) Cost per part by injected delay: exact teacher, teacher with $\pm$50\,ms jitter, decider D over HTTP; dots: episodes (10 per cell, 500 parts); dotted: 400\,ms deadline. (b) 340--400\,ms. (c) Late decisions: four episodes at once (preregistered), one at a time (exploratory).}
\label{fig:factory}
\end{figure}

Below the deadline the exact teacher paid nothing for delay: up to 380\,ms its cost stayed at 0.03 per part. At zero delay the decider cost 0.84 more per part than the teacher [0.64, 1.08] and let 11.2\% of defective parts through; this is the quality component. The decider itself began to pay for time before the nominal deadline. As preregistered, its live runs used four concurrent episodes on host 2, and under this load its round trips had a heavy tail (p99 66.7\,ms against a median of 3.7\,ms). On line A it was late on 6.0--13.1\% of parts at 340--370\,ms and on every part at 380\,ms, where its round trip and rounding to the next 20\,ms tick push answers past 400\,ms. In an exploratory rerun with one episode at a time, its mean late shares at 340--370\,ms fell to 0.2--0.4\% (0.2--0.8\% in single episodes) and its p99 to 4.1--4.3\,ms. The early time cost thus came mainly from running episodes concurrently on the host, the failure of concurrent evaluation reported in Section~\ref{sec:shortcuts}, seen here in a second environment (Section~\ref{app:from-factory}).

Past the deadline only time mattered: from 400\,ms every decision of both policies was late, both passed every part, and their costs were equal (7.52 per part). The interaction at 600\,ms, $-$0.84 [$-$1.07, $-$0.65], cancelled the quality component within the preregistered equivalence margin, so a better decider would not have helped. Timing noise also moves the edge for the teacher: with $\pm$50\,ms jitter it was already late on 30.4\% of parts at 360\,ms and cost more than the concurrent decider (2.06 against 1.83 per part; Figure~\ref{fig:factory}).

\subsection{Two shortcuts that hide the gap}
\label{sec:shortcuts}

\paragraph{Tick conversion} Most of the decider's calls finish within one tick (p50 2.14\,ms, p90 15.94\,ms), so tick conversion, which counts only whole ticks of delay, models almost no delay for it and predicts almost no loss. This is a prediction of the rule, not of asynchronous play. Under tick conversion it hit 99.2\% of its shots; on the same 30 raw seeds in asynchronous play it hit 84.1\%. Over the 18 distinct games of these seeds, the unit used everywhere else, the figures are 98.6\% and 83.8\%, a loss of 14.8\pp{} [10.6, 19.4]. Figure~\ref{fig:tick} shows the loss game by game. We had preregistered that a difference of 10\pp{} or more would retire tick conversion as a basis for real-time claims; on the preregistered unit, raw seeds, the difference was 15.1\pp.

Under tick conversion the teacher's skill falls in steps at tick boundaries (Section~\ref{app:from-dose}).

Concurrent evaluation on one host failed its preregistered calibration, so asynchronous evaluations ran one at a time (Section~\ref{app:from-concurrent}).

\section{Limitations}
\label{sec:limits}

\paragraph{Scope} The results come from two scenes of one game and one conveyor simulator, scripted teachers and one decider family, on two hosts. The simulator renders its own state into sentences; its results say nothing about camera-based inspection accuracy. A distant decider is represented only by the same decider acting on decisions 8 ticks old (v15), not by a separate model. Other games, human teachers and physical systems remain untested. The conveyor study ran four episodes concurrently, a setting that Section~\ref{sec:shortcuts} shows can change timing, so the decider's time cost below the deadline reflects that load.

\paragraph{Host variation} Skipped-tick shares varied between runs of the same arm. In earlier diagnostic asynchronous runs (Section~\ref{app:llm}), the zero-latency teacher skipped 6.3\% to 15.3\% of ticks, whereas it skipped about 0.1\% in the main studies. In two of the five v5 triads, the skipped-tick shares of the two deciders ranged from 11.1\% to 25.3\%. We have not identified the cause of this variation. Variation in skipped-tick shares between runs is the main reason the run-level intervals are wide. Sessions also differ: the same decider on the same 18 games hit 80.0\% of its shots in the decomposition study (Table~\ref{tab:decomp}) and 83.8\% in the tick-conversion comparison (Section~\ref{sec:shortcuts}), a session-to-session difference comparable to the DAgger gains.

\paragraph{Exogenous skipping} Proposition~1 requires the latency-matched teacher (S) and the decider (D) to face the same skip process. The studies on host 1 check this through the ratio of skipped-tick shares, not through identical schedules. The identical-schedule emulator (Section~\ref{sec:controlled}) put the bias term at $-$2.7\pp{} [$-$6.2, 0.5] for the original decider, outside our $\pm$2\pp{} tolerance. Its fidelity was confirmed on fresh seeds only where the wait after the last observation was short; on a heavily loaded host, both earlier fidelity checks failed because of long waits. The live study with one imposed schedule (Section~\ref{sec:live}) ran on host 2 only, which almost never skips ticks on its own; whether a schedule can be imposed as exactly on a host with its own timing noise is untested, and so is how sensitive the share-matched design is to differences in the length of skipped runs. There the DAgger decider kept a quality component of 2.8\pp{} [1.5, 4.1], whereas the emulator had put it at $-$0.2\pp{} [$-$2.5, 2.1] on 29 other games.

\paragraph{Unit count} The main studies use 18 distinct games, so their intervals are wide; the replication, the data-volume control, the input study, the host-2 study and the emulator's synthetic schedules use 56.

\paragraph{Other caveats} The underfitting explanation of Section~\ref{sec:diag} is not excluded: the first retrained decider matched only 74.0\% of its own round-1 training rows. The first two studies used a zero-latency teacher from an earlier session, 1.2\pp{} above the contemporaneous one, so their time components may be slightly overstated. The data-volume control compares its decider with O across sessions. The link from additive return, for which the bound of Section~\ref{sec:theory} holds, to hit accuracy is empirical; counted in hits per game, an additive quantity, the time and quality components are 2.5 and 3.3 hits in the decomposition study and 1.9 and 3.0 in the replication (exploratory). Each preregistration fixed its own primary hypotheses, and we did not correct for multiple comparisons across studies.

\paragraph{What the time component means} The time component measures the cost of waiting for the decider's server on the host where it is measured, including transport and host jitter. Host jitter is variation in timing on the computer running the experiment. It is not a lower bound for all policies, and a policy that anticipates its own delay could beat S.

\paragraph{Partial success} On-policy distillation reduced the quality component but did not remove it, and its effect over O is not significant at run level. We fixed two rounds in advance and did not tune them.

\section{Implications}
\label{sec:implications}

\paragraph{For evaluation} Most benchmarks for learned agents pause the environment while the agent decides. In our studies such a pause removes one of the two costs that a real-time agent pays, and the other cost, lower decision quality on the agent's own states, was of comparable size on one host (Section~\ref{sec:decomp}). A benchmark that reports only paused or tick-converted scores could therefore rank a slow, careful agent above a fast one that would do better on the real clock. We suggest that real-time claims report the time and quality components separately, on the host and schedule on which the agent is meant to run.

\paragraph{For system design} Large language models are now a common choice for agent decisions. In our comparators, language models that chose every action were much slower per decision than the distilled decider and played far worse (Section~\ref{app:llm}). The small distilled decider answered within a tick; its remaining loss is consistent with covariate shift, and on-policy distillation reduced it only in part. Where a late decision is replaced by a default, as on the conveyor, a better decider would not have helped past the deadline (Section~\ref{sec:factory}). This points to a division of labour that our results do not test directly: a large model can teach, check or handle rare cases away from the loop, while a fast learned decider acts inside it.

\paragraph{Scope} These are implications of two simulated environments and one decider family, not tested claims about other systems (Section~\ref{sec:limits}).

\section{Conclusion}
\label{sec:conclusion}

Section~\ref{sec:implications} set out what these results suggest for evaluation and system design; here we summarise them. In our setting a real-time gap can be split into a time component, the cost of waiting for decisions, and a quality component, the residual difference associated with the decider and its input representation; the two interact once skipping is large. For a fast distilled decider in asynchronous ViZDoom play, the two were comparable on one host; on another the time component vanished and the quality component did not.

When one delay schedule was imposed on every arm in live play, the quality component of the original decider remained (5.8\pp{} [4.0, 7.5] on 100 new games), and retraining reduced both the quality component and the decider's disagreement with the teacher on its own states. In the identical-schedule emulator the quality component persisted with no skipping, and matching only the share of skipped ticks may have biased the split. The quality component is consistent with covariate shift: the decider meets a different distribution of states in its own play than in training, although we have not excluded underfitting. Training on states produced by that play reduced the quality component on new games (two training seeds on each host), but only to a partial success under our preregistered criteria, and it did not help in a scene where the teacher is the weaker player. On the first host the reduction was greater than that obtained with the same amount of teacher-played data.

The measured cost of waiting differed between hosts and scenes. Where a late decision is replaced by a default, as on the conveyor, the time component behaved like a step: below the deadline, delay did not raise the exact teacher's cost, although host timing noise moved where the edge fell; past it both policies defaulted and the quality difference vanished. In our stale-decision emulation, acting on decisions 8 ticks old cost reward in one scene and was inconclusive in the other, and the best decision period differed between scenes.

A latency-matched arm, the teacher that waits for the decider before acting, indicates which of the two remedies to try, a faster decider or retraining on its own states, provided latency matching is adequate; equal skipped-tick shares alone do not establish this. The variation we saw between runs suggests that real-time evaluations should report run-level as well as game-level uncertainty. What carries over to other systems is the procedure, a latency-matched arm, a skip check and distinct games as the unit; the numbers are those of this game, these deciders and these hosts.

\section*{CRediT authorship contribution statement}
\textbf{Chihoon Shin}: Conceptualization, Methodology, Investigation, Formal analysis, Data curation, Visualization, Writing -- original draft, Writing -- review \& editing, Project administration. \textbf{Junyeong Lee}: Software, Resources, Validation. \textbf{Kihyeok Jeong}: Validation, Supervision, Writing -- review \& editing. \textbf{Wonok Kwon}: Methodology, Validation, Writing -- review \& editing.

\section*{Funding}
This research did not receive any specific grant from funding agencies in the public, commercial, or not-for-profit sectors.

\section*{Declaration of competing interest}
Chihoon Shin and Junyeong Lee are affiliated with the Myeongseongsimjae AX Institute, where the distilled decider evaluated in this study was developed. The other authors declare that they have no known competing financial interests or personal relationships that could have appeared to influence the work reported in this paper.

\section*{Data availability}
Preregistrations with their SHA-256 hashes, analysis scripts, decider weights and per-episode logs will be made publicly available upon publication; they are available to reviewers on request.

\section*{Declaration of generative AI and AI-assisted technologies in the manuscript preparation process}
During the preparation of this work the authors used Claude Opus 5 and, from 23 September 2026, Claude Opus 5.5 (Anthropic; Claude Code) and gpt-6-sol and gpt-6-astra (OpenAI; Codex command-line tool 0.156) in order to organise related literature, structure the manuscript, draft and edit text, and prepare figures from logged data. Their use in the research itself is summarised in Section~\ref{sec:protocol} and described in Section~\ref{app:from-ai}. No AI tool created or altered images of experimental observations: all figures are plotted by scripts from logged results or rendered by the game engine or the simulator's viewer. After using these tools, the authors reviewed and edited the content as needed and take full responsibility for the content of the publication.

\bibliographystyle{elsarticle-harv}
\bibliography{refs}

\clearpage
\appendix
\setcounter{figure}{0}\setcounter{table}{0}\setcounter{equation}{0}
\renewcommand{\thesection}{S\arabic{section}}
\renewcommand{\thefigure}{S\arabic{figure}}
\renewcommand{\thetable}{S\arabic{table}}
\renewcommand{\theequation}{S\arabic{equation}}
\section*{Supplementary material}

\begin{center}{\Large\bfseries Supplementary Material}\\[4pt]Separating Decision Time from Decision Quality in the Real-Time Gap of Distilled Deciders: Evidence from a Game and a Conveyor Simulator\end{center}

\section{Material moved from the main text}
\label{app:moved}

This section holds analyses and discussion that the main text summarises in a sentence.

\subsection{Further related work}
\label{app:from-related}

Agents that must choose the next action while the previous one is still executing face concurrent control \citep{xiao2020thinking}. Recent work formalises observational delay caused by the network's own inference \citep{anokhin2025handling}, random observation delays that arrive out of sequence \citep{karamzade2025modelbased}, the violation of the Markov assumption when an executed action depends on an earlier state \citep{zhu2026learning}, and the distribution shift created when the observation behind an action is stale by the time it executes \citep{dong2026reinforcement}; the remedies are state augmentation, history-augmented observations, architectural skip connections and fast action edits on top of a slow policy. Our study instead measures what a real host adds and separates its effect from that of decision quality.

\paragraph{Language-model game agents} Surveys single out low-latency control as a requirement of some genres \citep{hu2024}; benchmarks render games as text \citep{li2025,park2025}; and real-time reasoning environments change while the model is still thinking \citep{wen2025}. ViZDoom states have been represented in natural language before \citep{schwartz2019language}. Inference latency is a recognised obstacle for language-model agents in real-time games \citep{liu2023llmpowered,shao2024swarmbrain,wynter2024will}, and dual-process designs combine fast rules with asynchronous reasoning \citep{zhang2025leveraging}. We report language-model controllers only as peripheral comparators (Section~\ref{app:llm}).

\paragraph{Latency in deployed vision-language-action systems} Policies trained under idealised synchronous supervision can degrade severely under realistic inference latency, which has motivated injecting inference delays during imitation and reinforcement learning \citep{huang2026ticvla}. Such designs reduce the cost of latency without estimating how much of a measured loss it explains.

\subsection{Use of AI agents in the research}
\label{app:from-ai}

\paragraph{Use of AI agents in the research} Large language model agents (Claude Opus 5 until 23 September 2026 and Claude Opus 5.5 from that day, through the Claude Code environment; gpt-6-sol and gpt-6-astra through the Codex command-line tool 0.156) were used as research assistants: they proposed and critiqued experiment designs, implemented the experiment and analysis scripts, launched and monitored runs, and drafted analyses. To keep their output from entering the results unchecked, every experiment was frozen as a preregistration with a SHA-256 hash before its seeds were opened; every analysis script was run on positive and negative controls before it saw the results; every number reported in this paper is produced by a script from logged episode files and checked by an automated gate against those files; and the authors reviewed each design and result. Two errors introduced during the project (an agreement statistic that initially included decisions taken without ammunition, and an agent-produced draft that silently shortened the manuscript) were caught by these checks and corrected; they are reported in Supplementary Section~\ref{app:pilot}.

\paragraph{Use of AI agents} Large language model agents served as research assistants: they proposed and critiqued experiment designs, implemented the experiment and analysis scripts, launched and monitored runs, and drafted analyses. Their output entered the results only through fixed checks: every experiment was preregistered and hashed before its seeds were opened, every analysis script was run on positive and negative controls before it saw the results, and every reported number is produced by a script from logged episode files and checked by an automated gate against those files. The agents and versions used, and the errors these checks caught, are described in Section~\ref{app:from-ai}.

\subsection{Notes on the analysis}
\label{app:from-theory}

\paragraph{Interpretation of Proposition~\ref{prop:bias}} The Lipschitz condition requires changes in the teacher's metric to be no greater than a constant times changes in expected skipped-tick share. Proposition~\ref{prop:bias} turns the manipulation check into a bias term. The Lipschitz constant is not known; as a sensitivity estimate we plug in the slope fitted on all 450 episodes of the intervention study ($\hat\kappa \approx 0.61$ accuracy per unit skipped-tick share, common to all arms; Section~\ref{app:from-runvar}). In triad 1, the skipped-tick shares differ by 8.3\pp{}. Using the fitted slope, this difference corresponds to an estimated accuracy bias of about 5\pp{}. This is an estimate, not a certified bound, and matched mean skipped-tick shares do not guarantee matched skip processes; we therefore report skip-adjusted estimates beside, not instead of, the preregistered ones.

\paragraph{Two consequences} First, disagreements on skipped ticks cost nothing, because both policies must repeat $a_{t-1}$; only disagreements at decision ticks enter $\epsilon_D$. Second, the relevant error is measured on the decider's own asynchronous states, $d^D$, not on the teacher's; a decider can agree with the teacher on 98.4\% of teacher states and still carry an $\epsilon_D$ several times larger, which is what Table~\ref{tab:diag} shows.

\begin{table}[H]
\centering
\caption{Decider agreement with the teacher's actions}
\label{tab:diag}
\small
\begin{tabular}{@{}>{\raggedright\arraybackslash}p{0.46\linewidth}lrr@{}}
\toprule
State source & Decider & Agreement & Error \\
 & version & (\%) & rate (\%) \\ \midrule
Teacher demonstrations (validation split, 2,500 sampled) & original & 98.4 & 1.6 \\
On-policy, round 1 (24,547 decisions, 60 games) & original & 88.8 & 11.2 \\
On-policy, round 2 (25,529 decisions, 60 games) & after one round & 90.0 & 10.0 \\
\bottomrule

\end{tabular}
\par\smallskip{\footnotesize\raggedright Decisions with an enemy in view and ammunition left: 24,547 and 25,529 of the 44,820 and 47,905 decisions collected in the two rounds. On-policy states were collected in asynchronous play on seeds disjoint from evaluation.\par}
\end{table}

\paragraph{Ammunition filtering} Disagreements after the ammunition has run out have $Q^T_t(x,\pi_T(x))-Q^T_t(x,a)\approx 0$ for the firing part of the action, which is why we exclude them from the agreement statistic.

\paragraph{DAgger and its bound} This is the performance-difference argument behind DAgger \citep{ross2011}, applied to the augmented process with the skip process folded into the dynamics. DAgger trains on states drawn from $d^D$, so it targets the disagreement term of the bound in Proposition~\ref{prop:pdl}. It minimises a surrogate loss on aggregated data, which does not guarantee lower disagreement for the final policy on its own states, nor a lower actual gap; the bound also concerns an additive return, not accuracy. Whether the gap falls is the empirical question of Section~\ref{sec:dagger}.

\paragraph{Additive return versus accuracy} The proposition concerns an additive return such as expected hits. Our primary metric is accuracy, the ratio of hits to shots. The two move together here because the arms fire almost the same number of shots (Section~\ref{app:from-trims}). This link is empirical and is not part of the proposition.

\subsection{Further metrics, run-level variation and a skip-adjusted estimate}
\label{app:from-runvar}

\begin{figure}[!tb]
\centering
\includegraphics[width=\linewidth]{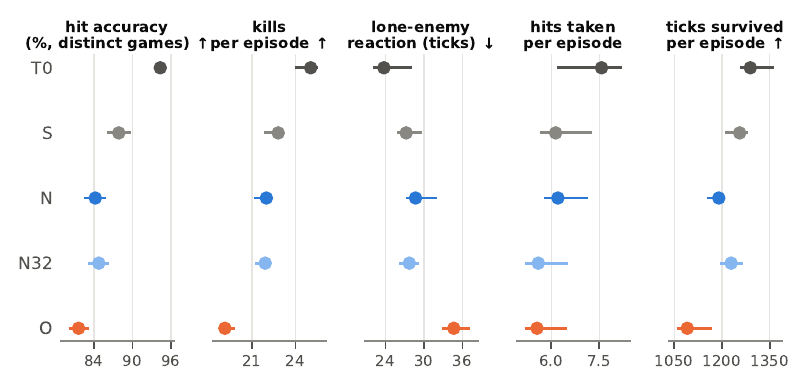}
\caption{Five metrics in the replication study (v6b; 100 seeds, 56 distinct games). Dots: means; bars: 95\% intervals from resampling distinct games. Lone-enemy reaction: ticks from the first tick with one enemy in view to the first shot aimed within 6 degrees of it. Hits taken depends on survival time and is not ranked.}
\label{fig:sm_metrics}
\end{figure}

Four further metrics in the replication study provide context for the accuracy result (Figure~\ref{fig:sm_metrics}). Kills per episode and reaction to a lone enemy follow the same broad teacher-versus-decider pattern: the original decider kills fewest enemies and reacts slowest. The zero-latency teacher survives about as long as S. Hits taken does not follow the ordering, because arms that survive longer are exposed to enemies for longer. The accuracy result is therefore not an artefact of firing fewer shots.

\begin{figure}[!tb]
\centering
\includegraphics[width=\linewidth]{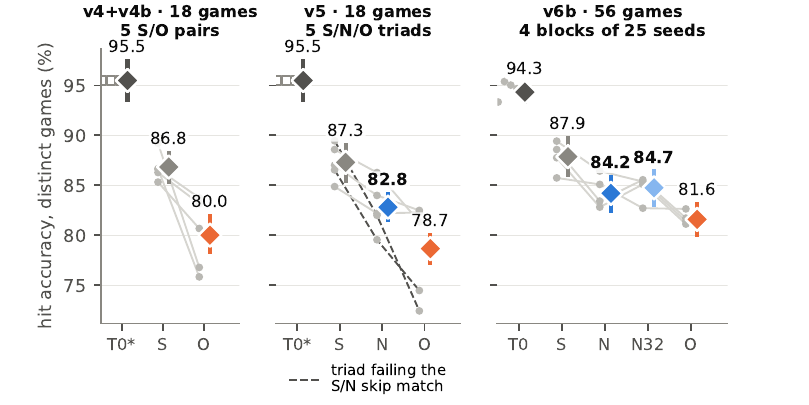}
\caption{Hit accuracy by run in the three studies on host 1: runs (v4+v4b: 5 S/O pairs; v5: 5 S/N/O triads; 18 games each) or 25-seed blocks (v6b: 4 blocks, 56 games). Squares: the three earlier-session T0 runs; filled diamonds: pooled, with 95\% game-bootstrap intervals; dashed: the two v5 triads outside the skip band 0.8--1.25.}
\label{fig:sm_studies}
\end{figure}

\paragraph{Host variation} The run-level interval of the improvement includes zero. In triads 1 and 2, both deciders skipped far more ticks than S (Figure~\ref{fig:sm_studies}, middle panel): N 15.7\% and 11.1\%, O 19.0\% and 25.3\%, S 7.4\% and 6.8\%. O skipped more than N, whichever order the arms ran in. These triads failed the S/N check (ratios 0.47 and 0.61) and also the S/O check (0.39 and 0.27); in triads 3--5 all three arms skipped 6.7--9.9\% of ticks. Excess skipping plausibly lowers a decider's accuracy, so these triads likely inflate both the quality component and the improvement over O. In triads 3--5 the improvement was 4.9, 1.5 and $-$0.1\pp{} and the remaining quality component 2.3, 3.0 and 2.7\pp; this split was not preregistered.

Figure~\ref{fig:sm_runs} shows how accuracy relates to skipping across runs in the three studies. The fitted slopes are negative for most arms, while the original decider in the replication study (v6b) has a small positive slope whose interval includes zero. Only three of the nine intervals exclude zero (N and O in the intervention study, v5, and N32 in v6b), and the negative slope of O in v5 is driven by the two triads that failed the skip match. With four or five runs per slope, and with skipping that may itself depend on the policy, these slopes are descriptive only.

\begin{figure}[H]
\centering
\includegraphics[width=\linewidth]{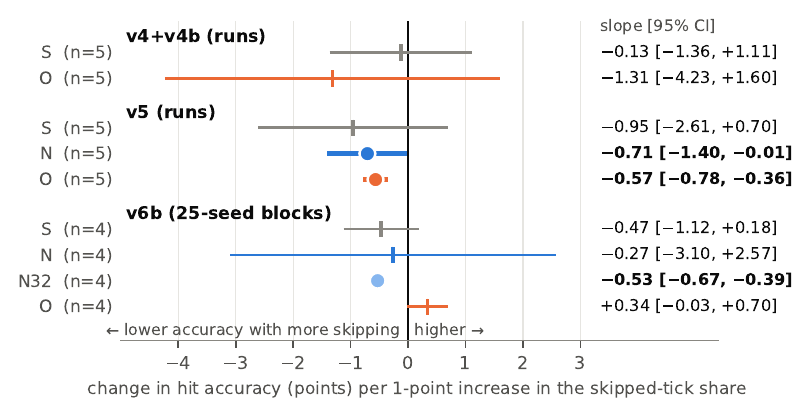}
\caption{Slope of hit accuracy on skipped-tick share, run by run, with 95\% $t$-intervals; n: runs (v4+v4b, v5: 18 games) or 25-seed blocks (v6b: 56 games). Filled circles: intervals that exclude zero. T0 is omitted because it skips almost no ticks. Scatter: Figure~\ref{fig:runs_scatter}.}
\label{fig:sm_runs}
\end{figure}

\begin{figure}[H]
\centering
\includegraphics[width=\linewidth]{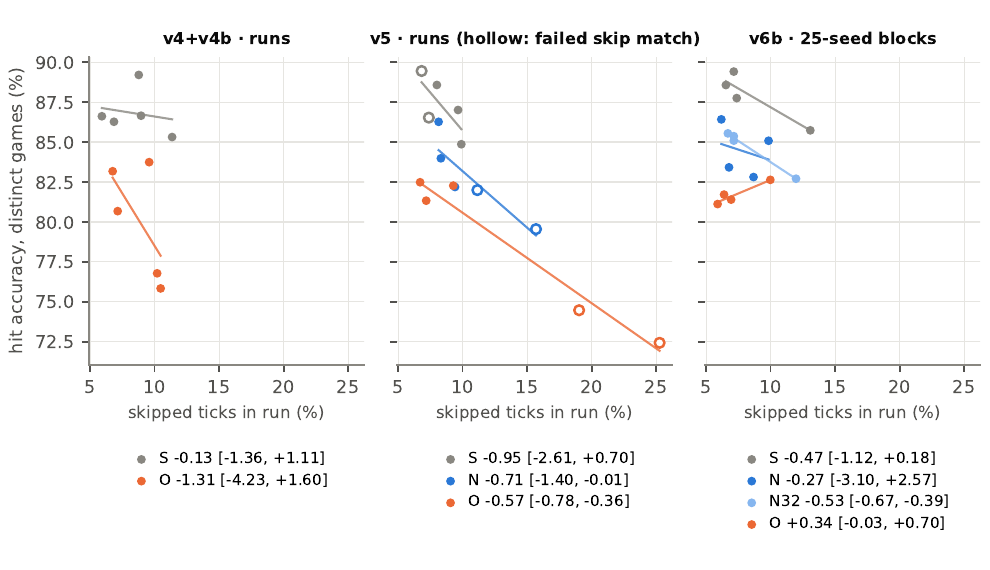}
\caption{Run-level data behind Figure~\ref{fig:sm_runs}: hit accuracy over distinct games against the skipped-tick share of each run (v4+v4b, v5) or 25-seed block (v6b), with least-squares lines per arm and the slope with its 95\% t-interval in each legend entry; hollow v5 markers are the triads that failed the S/N skip match. T0 is omitted because it skips almost no ticks.}
\label{fig:runs_scatter}
\end{figure}

\paragraph{Skip-adjusted estimate (exploratory)} Because skipping varies within as well as between runs, we regressed each episode's hit accuracy on arm indicators and the episode's skipped-tick share (450 episodes, weighted by shots, standard errors clustered by the 18 distinct games). The regression estimates an accuracy decrease of 6.1\pp{} for each additional 10\% of skipped ticks (standard error 0.4\pp). At equal skipping, N exceeds O by 3.9\pp{} (standard error 0.8) and falls 2.1\pp{} below S (standard error 1.0). This analysis was not preregistered, and a policy can influence its own skipping, so we treat it as supporting evidence only.

\subsection{Input and policy shares (v10)}
\label{app:from-v10}

\paragraph{Input and policy (v10)} The quality component compares a teacher reading exact coordinates with a decider reading a sentence. On host 1 we added a teacher SS that decides from the same sentence the decider reads (it agrees with the coordinate teacher on every decision outside the lowest ammunition bin). The input share, S $-$ SS, is $-$1.3\pp{} [$-$3.7, 1.1], below the preregistered 2\pp{} threshold; the policy share, SS $-$ N, is 3.2\pp{} [0.7, 5.7]. For accuracy we found no measurable input loss, and the policy share accounts for most of the quality component. The sentence does cost kills: SS does not fire with 1--4 bullets left and kills 2.96 fewer enemies per game [2.26, 3.62].

\subsection{Decision period and stale decisions (v15)}
\label{app:from-v15}

\paragraph{Decision period and stale decisions (v15)} A decider inside the game can choose how often to decide; a decider far from the game decides only as often as its round trip allows, and on a state that is already old. On host 2 we held each of O's actions for $k=0$, 1, 2, 4 or 8 ticks after deciding (arms H0 to H8), and added an arm L8 that acts every tick on the decision O made 8 ticks (228.6\,ms) earlier. The best period differs by scene (Figure~\ref{fig:holddose}). In defend\_the\_center (29 distinct games) deciding every tick is best: holding each action for two ticks lowers hit accuracy from 87.1\% to 28.3\% (H2 $-$ H0 $=$ $-$58.9\pp{} [$-$67.6, $-$51.4]). In deadly\_corridor (200 distinct games) holding for four ticks raises the mean reward from 58.8 to 105.2 (H4 $-$ H0 $=$ 46.4 [28.8, 64.2]), while holding for one tick lowers it (H1 $-$ H0 $=$ $-$21.9 [$-$35.6, $-$8.6]), against our preregistered prediction of a gain. A fresh decision held for 8 ticks (H8, a new decision every 8 ticks) against a decision 8 ticks old applied every tick (L8): H8 scores 35.1 points more than the stale decision in deadly\_corridor [19.3, 50.7]; in defend\_the\_center the difference is inconclusive ($-$7.8\pp{} [$-$17.4, 1.2]). A distant decider pays the stale-decision cost by construction; a fast local one can choose the period that suits the scene.

\begin{figure}[!tbp]
\centering
\includegraphics[width=\linewidth]{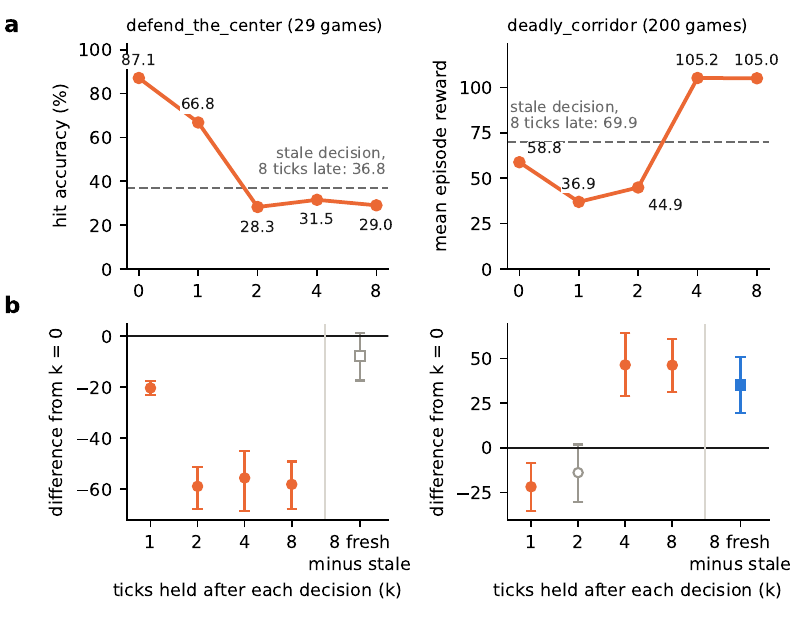}
\caption{How often to decide depends on the scene: O on host 2 holds each action for $k$ ticks (v15). (a) Hit accuracy (defend\_the\_center) and mean reward (deadly\_corridor) by $k$; dashed: O acting every tick on a decision 8 ticks (228.6\,ms) old. (b) Differences from $k=0$ and fresh versus stale (square), with 95\% game-bootstrap intervals.}
\label{fig:holddose}
\end{figure}

\subsection{Concurrent evaluation on one host}
\label{app:from-concurrent}

\paragraph{Concurrent evaluation on one host} To shorten wall-clock time, we preregistered a calibration of concurrent evaluation with each lane pinned to its own processor cores (core pinning). Three lanes, each with its own decider server, ran the latency-matched teacher at the same time. We would use the setup only if every lane stayed within 86.2--89.2\% raw accuracy and 5.5--12.1\% skipped ticks. These bands came from five solo runs. All three lanes failed: they skipped 82.5--83.2\% of ticks and hit 53.2--58.3\% of their shots. We did not separate whether pinning, concurrency or both caused this. On this host, asynchronous evaluations had to run one at a time.

\subsection{Further figures of the main studies}
\label{app:from-figs-main}

\begin{figure}[!tbp]
\centering
\includegraphics[width=\linewidth]{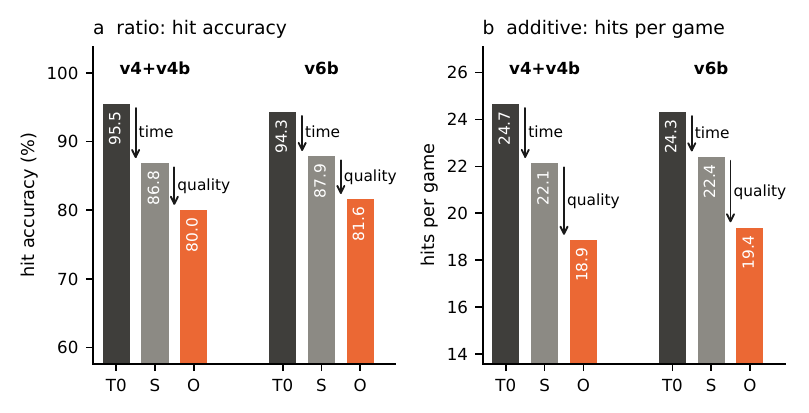}
\caption{The decomposition on two scales (exploratory for b). (a) Hit accuracy of T0, S and O in the decomposition study (v4+v4b; 95.5, 86.8, 80.0\%) and the replication (v6b; 94.3, 87.9, 81.6\%). (b) Hits per game in the same studies (24.7, 22.1, 18.9 and 24.3, 22.4, 19.4). Arrows mark the time and quality components.}
\label{fig:hits}
\end{figure}

\begin{figure}[!tbp]
\centering
\includegraphics[width=0.85\linewidth]{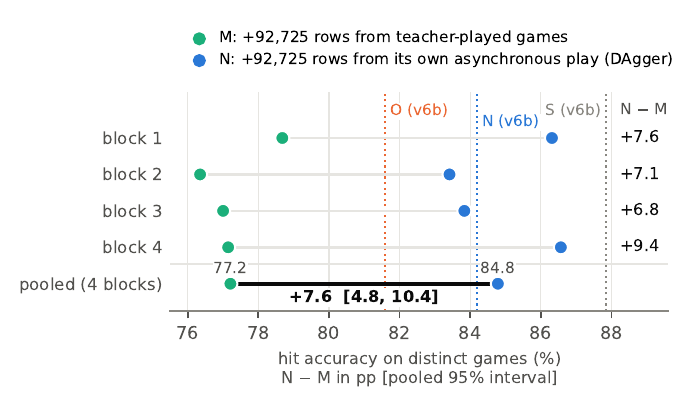}
\caption{The data-volume control (v7; 56 distinct games, 4 blocks): N adds 92,725 rows from its own asynchronous play, labelled by the teacher; M adds 92,725 rows from new teacher-played games. Gray segments: blocks; brackets: paired 95\% game-bootstrap interval of the pooled difference. Dotted: O, N and S from v6b; T0 lies off the axis at 94.3.}
\label{fig:volume}
\end{figure}

\begin{figure}[!tbp]
\centering
\includegraphics[width=\linewidth]{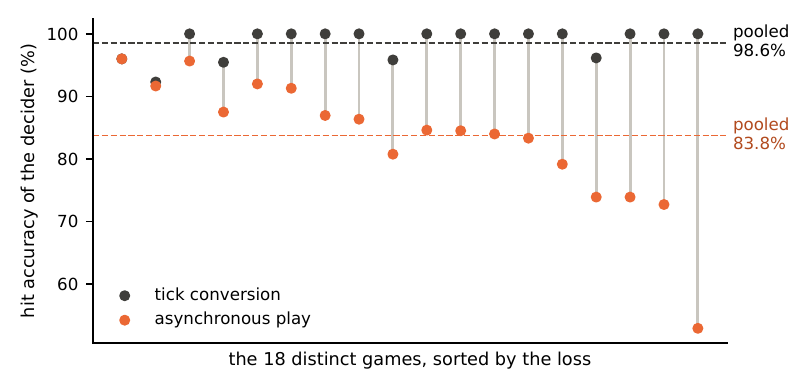}
\caption{Tick conversion against asynchronous play for the same decider on the 18 distinct games of seeds 850000--850029, sorted by the loss. Pooled hit accuracy: 98.6\% under tick conversion and 83.8\% in asynchronous play.}
\label{fig:tick}
\end{figure}

\subsection{Further diagnostics of the quality component}
\label{app:from-diag-more}

\begin{figure}[!tbp]
\centering
\includegraphics[width=\linewidth]{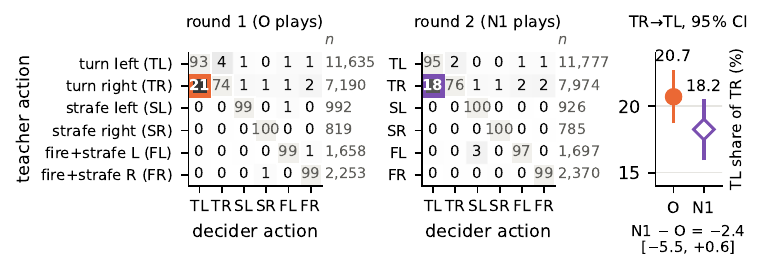}
\caption{On-policy errors in the DAgger rounds of v5 (enemy in view, ammunition left), driven by O and N1; row percentages (sums 99--100, errors on one 0--25\% scale). Right: TR$\to$TL rate, O 20.7\% [18.7, 22.7], N1 18.2\% [16.0, 20.5], difference $-$2.4 [$-$5.5, 0.6] (descriptive). No F row: the teacher never chose it.}
\label{fig:confusion}
\end{figure}

\subsection{Decision time on the two hosts}
\label{app:from-hosts-more}

\begin{figure}[H]
\centering
\includegraphics[width=\linewidth]{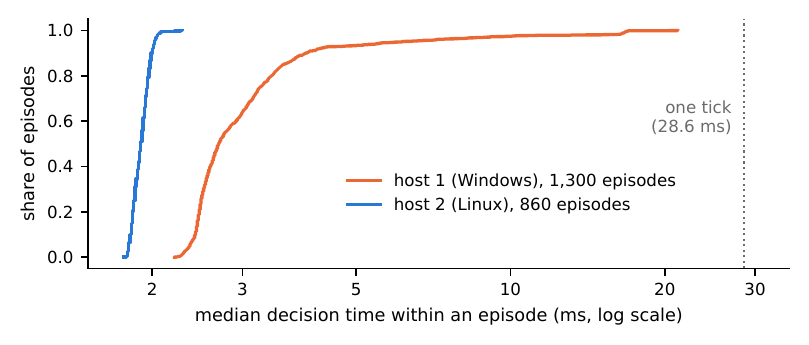}
\caption{Decision time on the two hosts (exploratory): cumulative share of episodes by the median decision time within the episode, 1,300 decider episodes on host 1 and 860 on host 2 (log scale). The dotted line marks one engine tick (28.6\,ms).}
\label{fig:latency_hosts}
\end{figure}

\subsection{Study ledger figure}
\label{app:from-protocol-more}

\begin{figure}[!tbp]
\centering
\includegraphics[width=\linewidth]{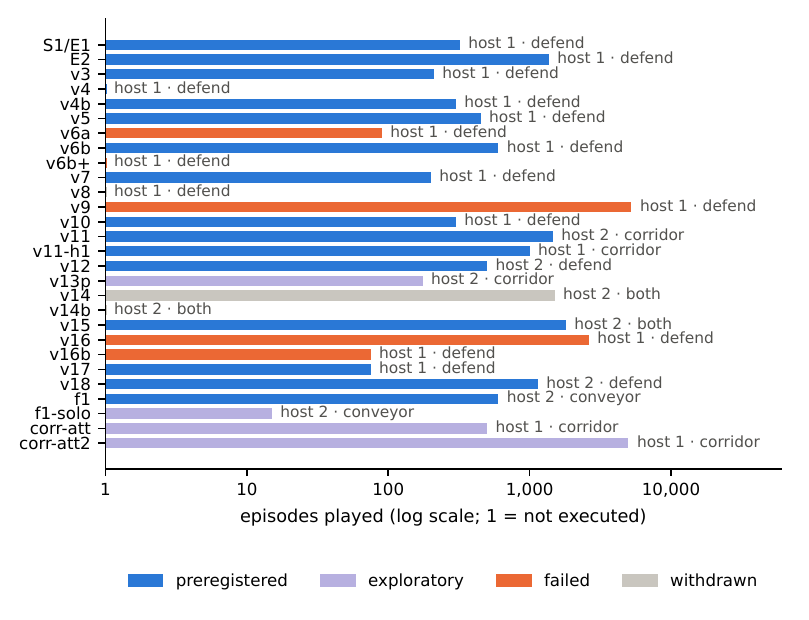}
\caption{All 28 studies of this paper (25,550 episodes in total): episodes played, host and scene. Blue: preregistered and run; lilac: exploratory or without a frozen hash; orange: failed its own calibration or fidelity check; gray: withdrawn. Bars at 1: not executed (v8, v14b) or counted in another row (v4 in v4b, v6b+ in v6b).}
\label{fig:ledger}
\end{figure}

\subsection{The two environments}
\label{app:from-setting-more}

\subsection{Identical schedules: a controlled-timing emulator}
\label{sec:controlled}

The live study of Section~\ref{sec:live} imposed one schedule in live play on host 2. To impose the recorded host-1 schedules exactly, and synthetic ones as well, we also rebuilt asynchronous play in the synchronous engine, where one schedule can be imposed on every policy. How the timing was rebuilt is described in Section~\ref{app:from-emulator}. With each policy deciding again on its own recorded schedule, hit accuracy matched in all 150 fresh episodes of two preregistered checks and all five game counts in 148, but the death tick differed in some episodes and both preregistered fidelity gates failed narrowly. The recorder had not logged the wait after the last observation, which the emulator then filled with the recorded gaps; in the second check, whose episodes skipped 19.4--39.4\% of ticks, every failing episode had such a wait of at least 5 ticks. A third preregistered check on fresh seeds (v17) logged that wait and imposed it: all five game counts matched and every emulated episode ended within its recorded window in 75 of 75 episodes. Its episodes skipped 8.1--9.8\% of ticks, as the v9 schedules used below did (7.9--11.3\%), and no wait exceeded 2 ticks, so fidelity is confirmed at the main studies' level of skipping but not for long waits. The results below come from the emulation rerun with that schedule (Section~\ref{app:controlled}).

On synthetic schedules imposed identically on T, O and N (Figure~\ref{fig:controlled}), a quality component remained without any skipping: O hit 10.0\pp{} less than the teacher [8.0, 12.2] and N 7.1\pp{} less [5.4, 9.0]. The time component rose with the skipped share, from 2.7\pp{} [1.0, 4.4] at 4.0\% of ticks skipped to 9.2\pp{} [7.2, 11.1] at 19.9\%, while N's quality component shrank from 7.1 to 2.1\pp: time and quality interact. DAgger improved on O on every schedule, by 2.5 to 4.5\pp.

\begin{figure}[H]
\centering
\includegraphics[width=0.92\linewidth]{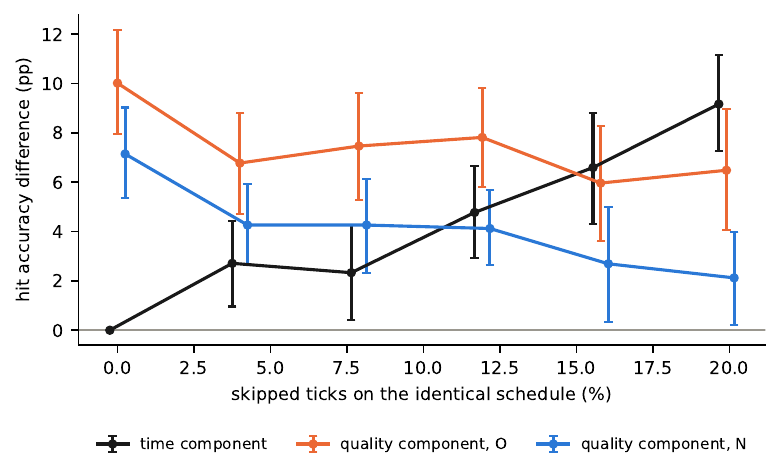}
\caption{Time and quality components on identical synthetic skip schedules imposed on T, O and N (56 distinct games, 100 seeds). Time: T without skipping minus T on the schedule; quality: T minus O and T minus N on the same schedule. Paired 95\% game-bootstrap intervals.}
\label{fig:controlled}
\end{figure}

On the schedules recorded on host 1 in v9 (Table~\ref{tab:controlled}), the teacher scored 2.7\pp{} less on the schedule S had met than on the original decider's own schedule [$-$6.2, 0.5], with 9.6\% and 8.6\% of ticks skipped. This bias term is not within our preregistered $\pm$2\pp{} tolerance, so matching shares did not rule it out. With schedules matched exactly, O's quality component was 7.1\pp{} [3.7, 10.6], against 4.4\pp{} [0.2, 8.9] under the share-matched design: the original decomposition may have understated the quality component and overstated the time component.

\begin{table}[H]
\centering
\caption{Recorded host-1 schedules replayed identically}
\label{tab:controlled}
\small
\begin{tabular}{@{}>{\raggedright\arraybackslash}p{0.46\linewidth}ll@{}}
\toprule
Contrast (host-1 recorded schedules, 29 games) & Original decider O & DAgger decider N \\
\midrule
Teacher on S schedule $-$ teacher on own schedule (bias term) & $-$2.7 [$-$6.2, 0.5] & $-$2.2 [$-$6.2, 1.6] \\
Quality, schedules matched exactly & 7.1 [3.7, 10.6] & $-$0.2 [$-$2.5, 2.1] \\
Quality, shares matched (original design) & 4.4 [0.2, 8.9] & $-$2.5 [$-$6.7, 1.8] \\
DAgger gain N $-$ O on N's schedule & \multicolumn{2}{l}{4.8 [1.8, 7.6]} \\
\bottomrule
\end{tabular}

\par\smallskip{\footnotesize\raggedright Schedules recorded on host 1 in v9 (seeds 860000--860049, 29 distinct games). Hit-accuracy differences in pp with paired game-bootstrap 95\% intervals. S: the teacher waiting for the decider's server; its schedule is the one the original design compared against.\par}
\end{table}

For the DAgger decider the picture differs. With schedules matched exactly, N's quality component on the recorded host-1 schedules was $-$0.2\pp{} [$-$2.5, 2.1], indistinguishable from zero, against 3.7\pp{} in the asynchronous replication on 56 games (Section~\ref{sec:replication}). Taken alone, this would move the intervention towards the preregistered strong-success criterion. We do not read it that way. First, it rests on the 29 distinct games of these seeds, and its interval still allows a residual of about 2\pp. Second, on the synthetic schedules over 56 games N kept a quality component at every skip rate, from 7.1\pp{} [5.4, 9.0] without skipping to 2.1\pp{} at 19.9\% (Figure~\ref{fig:controlled}). The recorded asynchronous episodes of the same 29 games point to the games rather than to the emulator: in them N's quality component was also indistinguishable from zero ($-$2.4\pp{} [$-$6.7, 1.8]; exploratory), while O's was the 4.4\pp{} above. The preregistered verdict therefore stays partial success. The two emulations agree on direction: DAgger improved on O on the recorded schedules (4.8\pp{} [1.8, 7.6]) and on every synthetic schedule. The live study of Section~\ref{sec:live}, run on other games and on host 2, put N's quality component under one imposed schedule at 2.8\pp{} [1.5, 4.1]; whether the emulator's lower value on these 29 games reflects the games or the emulation is open (Section~\ref{sec:limits}).

\subsection{Details shortened in the main text}
\label{app:from-trims}

Details that support single sentences of the main text, grouped by the section they belong to.

\paragraph{Introduction} A game agent that decides well but slowly and one that decides quickly but less well can look identical on a benchmark that pauses while they think. Fast decision services now make it cheap to put a learned decider inside a live loop \citep{golchinfar2026,park2025}.

\paragraph{Related work} Cloud-edge designs remove blocking synchronisation so that the controller always consumes the latest available feature \citep{peng2026latencytolerant}, and asynchronous deployment of world action models overlaps computation with execution and makes the alignment of commands with observation timestamps an engineering requirement \citep{team2026world}. Other approaches combine language-model planning with model-predictive control \citep{chen2026}, routing between fast and slow planners \citep{saha2024}, and speculative execution with a distilled corrector \citep{lin2025}. Distillation moves knowledge into cheaper models \citep{hinton2015,yu2024,kang2025}, and a 1.3M-parameter model trained on human demonstrations has been reported to outplay language models in defend\_the\_center \citep{golchinfar2026}. A time-conversion protocol that maps inference time to elapsed simulation time has been proposed for embodied agents \citep{zheng2025reflex}; as \citet{li2026think5hz} note, driving agents are commonly invoked on alternate simulation ticks with the previous command replayed in between; and the speed-accuracy trade-off of agents has been studied in real-time fighting games with local inference \citep{kang2025winfast}. Performance is better reported as a distribution than as a point estimate \citep{clary2019letsplay}, comparisons need statistical tests with sample sizes set from the expected effect \citep{colas2019a}, reliability has its own metrics of variability and worst-case risk \citep{chan2019measuring}, and offline and online evaluation create different pressures \citep{patterson2023empirical}.

\paragraph{Results} Our first preregistration (three pairs) asked whether the decider fell more than 5\pp{} below S in all three pairs; it did not, and the verdict was ``mixed''. Its three pairs are the first three of the five-pair study (v4b). The three-pair estimate had suggested that about three quarters of the gap was time; the two added pairs reversed that reading, and we report both. In the intervention study, N does not trade accuracy for fewer shots: it fires 25.6 shots per game against 25.7 for S and 23.5 for O, and kills 21.5 enemies per game against 22.5 for S and 18.4 for O. In the data-volume control, the difference between N and M has a game-bootstrap interval of 4.8--10.4\pp{} and an interval over the four blocks of 5.9--9.6\pp; the difference lies between 6.8 and 9.4\pp{} in every block. In the replication, the time component, now measured against a contemporaneous zero-latency teacher, is 6.5\pp{} (4.6--8.3). On host 2 the decider's accuracy rose (O 88.6\% against 81.6\%), but its distance from a teacher that waits as long as it does hardly changed. In deadly\_corridor on host 1 (94 distinct games) S skipped 10.0\% of ticks against 1.4\% for T0 and scored 12.1 points \emph{more} than T0 ([$-$0.6, 25.7]): in this scene a held forward step keeps earning progress, so waiting helps. In the controlled-timing emulator, on N's own recorded schedule, DAgger improved on O by 4.8\pp{} [1.8, 7.6]. In the conveyor simulator, cost per part weights an escaped defect 100, a false rejection 20, a line stop 10, an escalation 5 and a rework 1 (Figure~\ref{fig:factory}). When the belt waited for each decision, cost did not change with delay, but the decider's throughput fell from 148.7 to 59.8 parts per minute at 600\,ms, against 249.7 without waiting. In the tick-conversion comparison, the decider also killed fewer enemies per game in asynchronous play than under tick conversion (19.90 against 23.87) and reacted to a lone enemy more slowly (32.8 ticks against 24.9).

\paragraph{Limitations} A preregistered attempt to impose identical schedules by recording and replaying them (v9) failed its own fidelity check and is reported in Section~\ref{app:v9}; it is not used as evidence.

\section{Proofs}
\label{app:proofs}

Throughout, the horizon $H$ is finite, rewards lie in $[0,1]$, and an episode that ends early is extended with an absorbing zero-reward state, so every sum runs over $t=0,\dots,H-1$.

\paragraph{Proof of Proposition~\ref{prop:bias}} Write $J_{T}(Z) = J(\pi_T,Z)$ and $J_D(Z)=J(\pi_D,Z)$. The decomposition $J_T(Z_0)-J_D(Z_D) = [J_T(Z_0)-J_T(Z_D)] + [J_T(Z_D)-J_D(Z_D)]$ is an identity. In the first bracket the policy is fixed at $\pi_T$ and only the skip process changes; in the second the skip process is fixed at $Z_D$ and only the policy changes. Assumption~\ref{ass:exo} is what makes $Z_D$ a well-defined object shared by the two policies: without it, the skip process experienced by a policy would depend on the policy's own actions and the second bracket would mix a policy change with a skip change. If S runs under $Z_S$, the estimator of the quality component is
\[
\hat Q = J_T(Z_S) - J_D(Z_D) = \underbrace{J_T(Z_D) - J_D(Z_D)}_{Q} + \big[J_T(Z_S) - J_T(Z_D)\big].
\]
If $J_T$ depends on the skip process only through $\sigma(Z)$ up to a $\kappa$-Lipschitz function, $|J_T(Z_S)-J_T(Z_D)| \le \kappa\,|\sigma(Z_S)-\sigma(Z_D)|$, which is the stated bound. The same argument applied to the time component with $Z_0$ and $Z_S$ in place of $Z_0$ and $Z_D$ bounds the bias of the time component by $\kappa\,|\sigma(Z_S)-\sigma(Z_D)|$ with the opposite sign, so the two biases cancel in the total gap. \qed

\paragraph{Performance difference in the augmented process} Fix $Z$ and let $V^T_t(x) = \mathbb{E}\big[\sum_{u\ge t} r(s_u,a_u)\mid x_t=x\big]$ under $\pi_T$, and $Q^T_t(x,a) = r(s,a) + \mathbb{E}[V^T_{t+1}(x_{t+1}) \mid x_t=x, a_t=a]$, with $V^T_H \equiv 0$. Because $\pi_T$ is deterministic, $V^T_t(x) = Q^T_t(x,\pi_T(x))$.

\begin{lemma}
\label{lem:pdl}
For any policy $\pi_D$ in the augmented process, $J(\pi_T,Z)-J(\pi_D,Z) = \sum_{t=0}^{H-1}\mathbb{E}_{x\sim d^D_t}\big[Q^T_t(x,\pi_T(x)) - Q^T_t(x,\pi_D(x))\big]$.
\end{lemma}
\begin{proof}
Let $x_0,a_0,x_1,\dots$ be a trajectory of $\pi_D$. Then $J(\pi_D,Z) = \mathbb{E}\big[\sum_t r(s_t,a_t)\big]$ and, since $x_0$ has the same law under both policies, $J(\pi_T,Z) = \mathbb{E}[V^T_0(x_0)]$. Telescoping,
\[
J(\pi_T,Z) - J(\pi_D,Z) = \mathbb{E}\Big[\sum_{t}\big(V^T_t(x_t) - r(s_t,a_t) - V^T_{t+1}(x_{t+1})\big)\Big],
\]
because $\sum_t (V^T_t(x_t)-V^T_{t+1}(x_{t+1})) = V^T_0(x_0) - V^T_H(x_H) = V^T_0(x_0)$. Conditioning each term on $(x_t,a_t)$ with $a_t=\pi_D(x_t)$ and using the Markov property of the augmented state, $\mathbb{E}[r(s_t,a_t)+V^T_{t+1}(x_{t+1})\mid x_t,a_t] = Q^T_t(x_t,\pi_D(x_t))$. Substituting $V^T_t(x_t)=Q^T_t(x_t,\pi_T(x_t))$ gives the claim.
\end{proof}

\paragraph{Proof of Proposition~\ref{prop:pdl}} Consider one term of Lemma~\ref{lem:pdl}. If $z_t=0$, both policies are constrained to $a_{t-1}$, which is part of $x_t$, so $\pi_T(x_t)=\pi_D(x_t)$ and the term is zero. If $z_t=1$ and $\pi_D(x_t)=\pi_T(x_t)$, the term is zero. Otherwise it is at most $\Delta_Q$ by definition. Hence each term is at most $\Delta_Q\,\Pr_{x\sim d^D_t}[z_t=1,\pi_D(x)\ne\pi_T(x)]$, and summing over $t$ gives $H\,\epsilon_D\,\Delta_Q$. If $\pi_T$ is greedy with respect to $Q^T$, every term is non-negative. \qed

\paragraph{Remarks} (i) The bound uses the decider's own state distribution $d^D$, which under asynchronous play includes the states produced by skipped ticks; agreement measured on the teacher's distribution does not enter it. (ii) $\Delta_Q$ can be much smaller than $H$: a wrong turn at one decision tick is usually corrected at the next, so its effect on the teacher's value is local. Disagreements whose action difference has no effect on the reward (firing without ammunition) contribute nothing. (iii) Proposition~\ref{prop:pdl} is an upper bound: it names the disagreement term that DAgger targets, but a smaller bound does not imply a smaller gap; whether the gap falls is shown by the experiments (Sections~\ref{sec:replication} and~\ref{sec:live}).

\section{Per-run results}
\label{app:runs}

Tables~\ref{tab:runs_v4b} and~\ref{tab:runs_v5} list every asynchronous evaluation run of the decomposition and intervention studies with all recorded metrics. ``Acc.\ dist.'' is hit accuracy over distinct games, ``Acc.\ raw'' over raw seeds; kills, shots, hits taken and ticks survived are per-episode means over raw seeds; ``Lone react.'' is the mean reaction to a lone enemy in ticks; ``Skip'' is the share of engine ticks skipped. Runs are listed in execution order.

\begin{table}[H]
\centering
\caption{Decomposition study (v4, v4b): five S/D pairs in execution order, seeds 850000--850029.}
\label{tab:runs_v4b}
\scriptsize\setlength{\tabcolsep}{2pt}
\begin{tabular}{lrrrrrrrr}
\toprule
Run & Acc. dist. & Acc. raw & Kills & Shots & Hits taken & Ticks & Lone react. & Skip \% \\ \midrule
S1 & 85.3 & 87.3 & 22.50 & 25.7 & 5.53 & 1250 & 29.2 & 11.3 \\
D4 & 80.7 & 81.6 & 19.37 & 23.7 & 5.20 & 1112 & 35.0 & 7.1 \\
S2 & 86.6 & 87.9 & 22.87 & 26.0 & 6.10 & 1235 & 25.5 & 5.9 \\
D5 & 83.7 & 81.9 & 18.83 & 23.0 & 6.00 & 1072 & 38.7 & 9.6 \\
S3 & 86.3 & 86.7 & 22.13 & 25.5 & 5.53 & 1243 & 25.9 & 6.9 \\
D6 & 83.2 & 82.4 & 19.17 & 23.3 & 6.03 & 1087 & 36.0 & 6.7 \\
S4 & 89.2 & 89.7 & 23.17 & 25.8 & 6.17 & 1273 & 26.4 & 8.8 \\
D7 & 76.8 & 78.0 & 18.13 & 23.2 & 5.47 & 1063 & 40.8 & 10.2 \\
S5 & 86.7 & 88.4 & 22.30 & 25.2 & 5.23 & 1207 & 26.4 & 8.9 \\
D8 & 75.8 & 77.8 & 18.40 & 23.7 & 5.53 & 1066 & 36.4 & 10.5 \\
\bottomrule
\end{tabular}

\end{table}

\begin{table}[H]
\centering
\caption{Intervention study (v5): five triads in execution order (triads 2 and 4 in reverse order), seeds 850000--850029.}
\label{tab:runs_v5}
\scriptsize\setlength{\tabcolsep}{2pt}
\begin{tabular}{lrrrrrrrr}
\toprule
Run & Acc. dist. & Acc. raw & Kills & Shots & Hits taken & Ticks & Lone react. & Skip \% \\ \midrule
S1 & 86.5 & 87.2 & 22.47 & 25.8 & 5.53 & 1234 & 27.0 & 7.4 \\
N1 & 79.5 & 80.3 & 20.50 & 25.5 & 5.57 & 1159 & 32.9 & 15.7 \\
O1 & 74.5 & 74.2 & 17.60 & 23.7 & 5.70 & 1050 & 35.1 & 19.0 \\
O2 & 72.4 & 71.1 & 16.10 & 22.6 & 5.17 & 997 & 38.8 & 25.3 \\
N2 & 82.0 & 83.3 & 21.37 & 25.6 & 5.47 & 1192 & 27.2 & 11.1 \\
S2 & 89.5 & 88.2 & 22.40 & 25.4 & 6.00 & 1245 & 30.1 & 6.8 \\
S3 & 88.6 & 87.8 & 22.63 & 25.8 & 5.93 & 1247 & 26.0 & 8.0 \\
N3 & 86.3 & 87.1 & 22.37 & 25.7 & 5.43 & 1164 & 28.0 & 8.1 \\
O3 & 81.3 & 82.0 & 19.43 & 23.7 & 5.13 & 1081 & 31.7 & 7.2 \\
O4 & 82.5 & 82.1 & 19.47 & 23.7 & 5.77 & 1084 & 34.0 & 6.7 \\
N4 & 84.0 & 85.1 & 21.67 & 25.5 & 6.23 & 1191 & 28.0 & 8.3 \\
S4 & 87.0 & 87.9 & 22.83 & 26.0 & 5.80 & 1314 & 28.4 & 9.6 \\
S5 & 84.9 & 86.1 & 22.10 & 25.6 & 5.40 & 1195 & 27.4 & 9.9 \\
N5 & 82.2 & 84.0 & 21.70 & 25.8 & 6.07 & 1163 & 29.8 & 9.4 \\
O5 & 82.3 & 82.2 & 19.57 & 23.8 & 5.10 & 1092 & 37.5 & 9.3 \\
\bottomrule
\end{tabular}

\end{table}

\begin{table}[H]
\centering
\caption{Skip-adjusted contrasts per triad of the intervention study (exploratory). Adjusted contrasts move each decider to the skipped-tick share of the arm it is compared with using the pooled slope $\kappa$ of Section~\ref{app:from-runvar} (Proposition~\ref{prop:bias}). Values in percentage points, skipped-tick shares in percent.}
\label{tab:adjusted}
\scriptsize
\begin{tabular}{lrrrrrrr}
\toprule
Triad & Skip S & Skip N & Skip O & S$-$N & S$-$N adj. & N$-$O & N$-$O adj. \\ \midrule
1 & 7.4 & 15.7 & 19.0 & 7.0 & 1.9 & 5.1 & 3.0 \\
2 & 6.8 & 11.1 & 25.3 & 7.5 & 4.8 & 9.6 & 1.0 \\
3 & 8.0 & 8.1 & 7.2 & 2.3 & 2.2 & 4.9 & 5.5 \\
4 & 9.6 & 8.3 & 6.7 & 3.0 & 3.8 & 1.5 & 2.5 \\
5 & 9.9 & 9.4 & 9.3 & 2.7 & 3.0 & -0.1 & -0.0 \\
\bottomrule
\end{tabular}

\end{table}

\begin{figure}[H]
\centering
\includegraphics[width=0.8\linewidth]{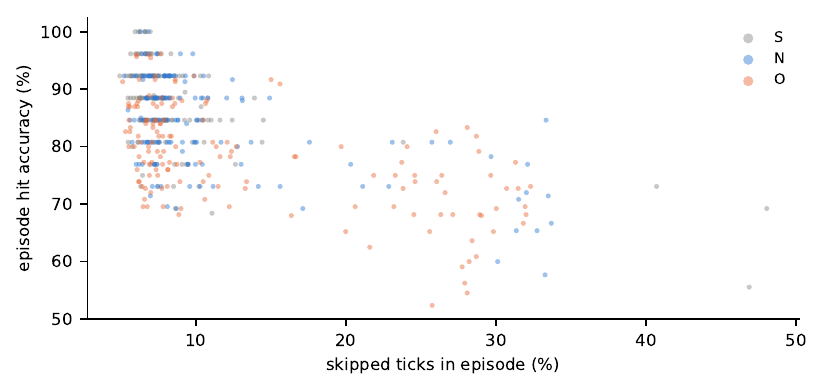}
\caption{Episode hit accuracy against the share of skipped ticks in that episode for the 450 episodes of the intervention study.}
\label{fig:skip_scatter}
\end{figure}

\section{DAgger collection and training}
\label{app:dagger}

Table~\ref{tab:collection} summarises the two collection rounds and Figure~\ref{fig:training} the training curves. Table~\ref{tab:confusion} gives the confusion between the teacher's action and the decider's action, as a percentage of each teacher action. It covers the decider's own asynchronous states: decisions with an enemy in view and ammunition left. Most on-policy errors are turns in the wrong direction when the teacher turns right; strafing and firing decisions are almost always correct.

\begin{table}[H]
\centering
\caption{DAgger collection rounds. ``Engaged'' counts decisions with an enemy in view.}
\label{tab:collection}
\scriptsize
\begin{tabular}{lllrrrr}
\toprule
Round & Driver & Seeds & Rows & Engaged & Skip \% & Minutes \\ \midrule
1 & a814587e03 & 720000--720059 & 44,820 & 43,707 & 11.5 & 31 \\
2 & 240805651d & 720100--720159 & 47,905 & 46,725 & 11.1 & 33 \\
\bottomrule
\end{tabular}

\end{table}

\begin{table}[H]
\centering
\caption{On-policy confusion (rows: teacher action, columns: decider action, \% of row). TL/TR turn left/right, F fire, SL/SR strafe, FSL/FSR fire while strafing.}
\label{tab:confusion}
\scriptsize
\begin{tabular}{lrrrrrrrr}
\toprule
Teacher & $n$ & TL & TR & F & SL & SR & FSL & FSR \\ \midrule
\multicolumn{9}{l}{\emph{Round 1 (original decider in control)}, 24,547 decisions} \\
TL & 11,635 & 93.2 & 4.0 & 0.0 & 0.6 & 0.4 & 1.0 & 0.9 \\
TR & 7,190 & 20.7 & 73.5 & 0.0 & 1.3 & 1.2 & 1.3 & 2.0 \\
SL & 992 & 0.2 & 0.2 & 0.0 & 98.8 & 0.1 & 0.7 & 0.0 \\
SR & 819 & 0.2 & 0.2 & 0.0 & 0.0 & 99.5 & 0.0 & 0.0 \\
FSL & 1,658 & 0.1 & 0.1 & 0.0 & 0.2 & 0.0 & 99.1 & 0.5 \\
FSR & 2,253 & 0.2 & 0.0 & 0.0 & 0.0 & 0.8 & 0.0 & 98.9 \\
\midrule
\multicolumn{9}{l}{\emph{Round 2 (decider after one round in control)}, 25,529 decisions} \\
TL & 11,777 & 95.5 & 2.4 & 0.0 & 0.3 & 0.4 & 0.7 & 0.7 \\
TR & 7,974 & 18.2 & 75.7 & 0.0 & 1.1 & 0.9 & 1.9 & 2.1 \\
SL & 926 & 0.0 & 0.0 & 0.0 & 99.9 & 0.0 & 0.1 & 0.0 \\
SR & 785 & 0.0 & 0.0 & 0.0 & 0.0 & 99.9 & 0.0 & 0.1 \\
FSL & 1,697 & 0.3 & 0.1 & 0.0 & 2.7 & 0.0 & 96.8 & 0.1 \\
FSR & 2,370 & 0.0 & 0.1 & 0.0 & 0.1 & 0.4 & 0.1 & 99.2 \\
\bottomrule
\end{tabular}

\end{table}

\begin{figure}[H]
\centering
\includegraphics[width=0.8\linewidth]{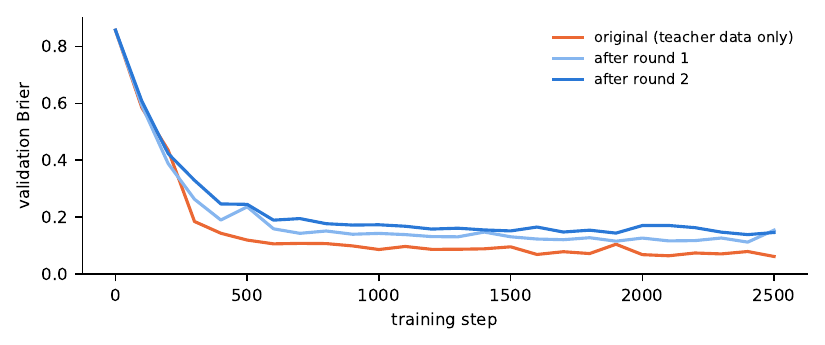}
\caption{Validation Brier score during training for the original decider and after each DAgger round (seed 31). The validation split of the later rounds includes on-policy rows.}
\label{fig:training}
\end{figure}

\section{Evaluation loop of the tick-converted evaluations}
\label{app:loopsupp}
\begin{figure}[H]
\centering
\includegraphics[width=\linewidth]{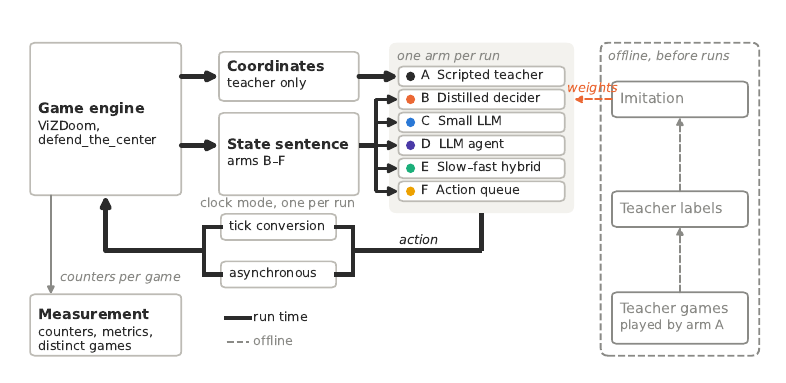}
\caption{The evaluation loop of the earlier tick-converted evaluations, including the language-model comparator arms (labelled A--F there; in the main text D denotes the original decider). The engine sends exact coordinates to the scripted teacher and a state sentence, written by a perception script, to the learned arms. One arm plays each run, under tick conversion (which waits for each decision and holds the previous action for $m(L)$ ticks) or asynchronous play on wall-clock time.}
\label{fig:loop}
\end{figure}

\section{Decider network}
\label{app:decider}
\begin{figure}[H]
\centering
\includegraphics[width=\linewidth]{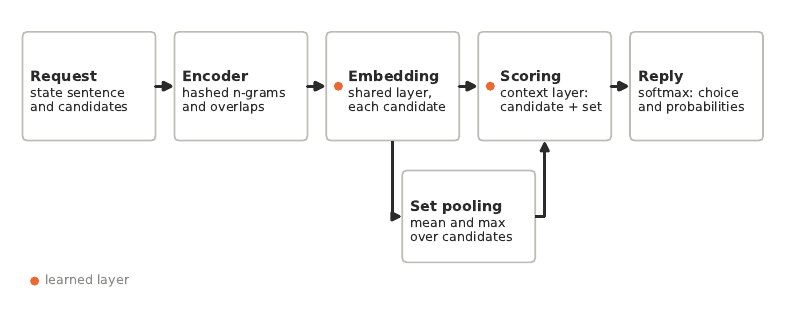}
\caption{The decider scores each candidate action against the state sentence. A request carries the state sentence and the seven candidate actions with their descriptions. The encoder turns each candidate, with the sentence, into 1,785 hashed word and character n-gram and overlap features; a shared layer (1,785 to 128, ReLU) embeds each candidate; set pooling (mean and max, 256) summarises all candidates; a context layer (the candidate's 128 values with the 256 pooled values, 384 to 96, ReLU) and a linear score (96 to 1) rate each candidate; a softmax with a fitted temperature returns the choice and a probability for every candidate. 265,665 parameters; about 2\,ms per decision on one CPU core; the same code serves both scenes and only the weights change. Training uses the teacher's labels with a sampled policy-gradient estimate of the Brier score, keeps the step with the best validation Brier score and fits the temperature on a calibration split. No pretrained language model is used.}
\label{fig:decider}
\end{figure}

\section{Examples of on-policy errors}
\label{app:qual}

\section{Latency dose-response under tick conversion}
\label{app:dosesupp}

\subsection{Tick-converted latency dose response}
\label{app:from-dose}

Under tick conversion, latency lowers skill in steps at tick boundaries, a property of the rule that real play need not follow. When the teacher is given a fixed latency under tick conversion, its skill holds and then drops at discrete latencies (Figure~\ref{fig:dose}). With 0, 10 or 20\,ms it hits 99.5\% of its shots, kills 25.03 enemies per game and reacts to a lone enemy in 25.4 ticks. At 30\,ms accuracy falls to 87.8\%, kills to 16.97 and reaction rises to 42.5 ticks; 50\,ms produces exactly the same numbers. Accuracy is 44.4\% at 100\,ms, 16.1\% at 200\,ms and zero at 500\,ms. Any latency under one tick gives $m=0$, and 30 and 50\,ms both give $m=1$.

\begin{figure}[!tbp]
\centering
\includegraphics[width=\linewidth]{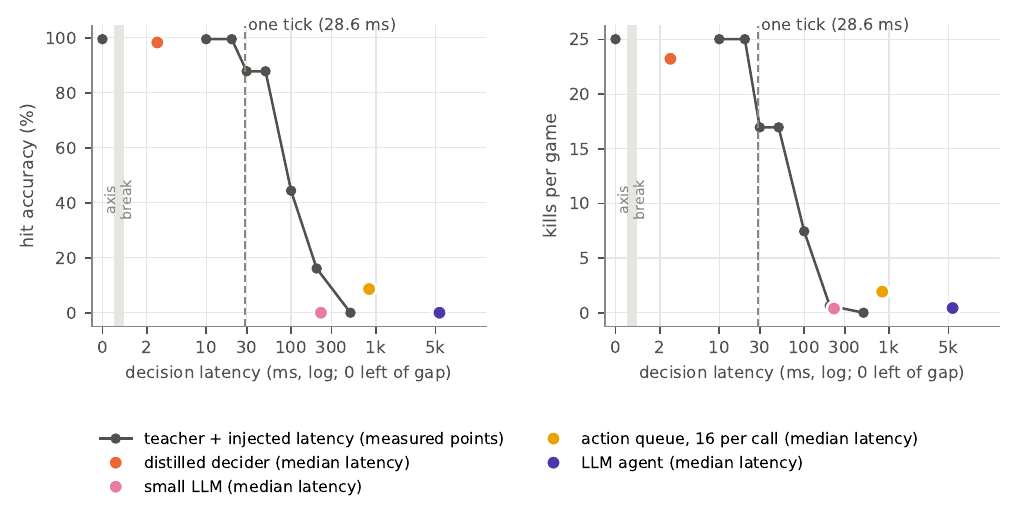}
\caption{Under tick conversion, injected latency leaves the teacher's skill intact below one tick and removes it quickly above, and slower controllers fall on the degraded side. Comparator arms are placed at the median of their per-episode median latency. They are context rather than latency-only comparisons, because their policies differ. The arms are the distilled decider; a 1.7B-parameter open-weight large language model (LLM) choosing every action (small LLM); the same model returning 16 actions per call (action queue); and a command-line language-model agent (LLM agent), on the 20 raw seeds of the comparator evaluation. The gray band is an axis break: 0 ms is plotted to its left. Teacher with fixed injected latency (one run per latency on the 100 seeds of the second tick-converted evaluation; points are measured values joined by lines, not a fitted step); These tick-converted runs were preregistered as single runs without intervals.}
\label{fig:dose}
\end{figure}

\section{Duplicated seeds}
\label{app:seeds}
Running the deterministic teacher on 200 seeds produced 103 distinct games, and a single game accounted for 98 of the seeds; two further ranges of 100 seeds gave 47 and 58 distinct games (Figure~\ref{fig:seeds}). The 30 evaluation seeds of the main studies give 18, and the 100 seeds of the replication give 56. Intervals computed over seeds treat copies of one game as independent evidence \citep{colas2018}.

\begin{figure}[H]
\centering
\includegraphics[width=0.6\linewidth]{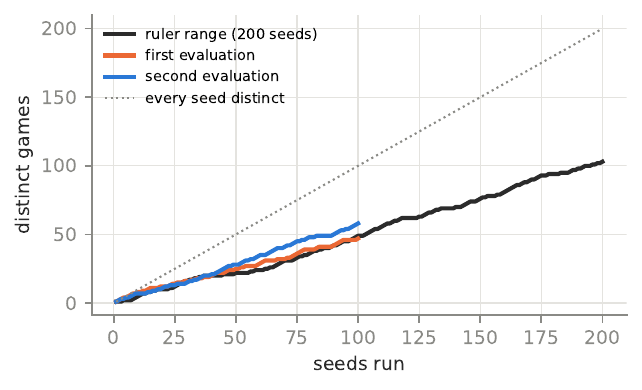}
\caption{Distinct teacher games as seeds are added. The dotted line is the case in which every seed gives a new game.}
\label{fig:seeds}
\end{figure}

\section{Metric qualification}
\label{app:ruler}
Before training, we checked that the event metrics separate scripts known to differ (Figure~\ref{fig:ruler}). On 200 seeds the teacher hit 99.7\% of its shots and reacted in 25.3 ticks; a variant without strafing hit 99.8\% and reacted in 22.3 ticks; a variant that fires whenever the target is within 15 degrees hit 31.9\% and reacted in 45.2 ticks; the teacher with 100\,ms latency hit 43.1\% and reacted in 104.6 ticks, and with 500\,ms it hit nothing.

\begin{figure}[H]
\centering
\includegraphics[width=\linewidth]{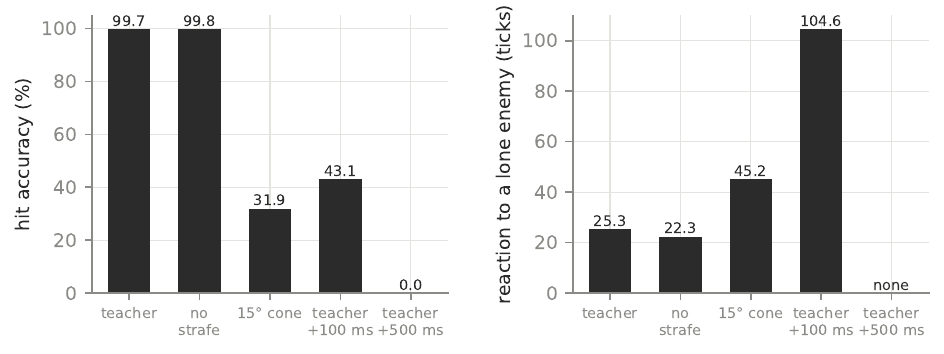}
\caption{Qualifying the metrics on 200 seeds before any learning.}
\label{fig:ruler}
\end{figure}

\section{Pilot measurement and corrections}
\label{app:pilot}
The DAgger preregistration cited a pilot on two design games in which the original decider agreed with the teacher on 69.7\% of its own engaged decisions against 95.8\% on teacher states. That pilot included decisions taken with no ammunition left, where the teacher's label has no effect on accuracy; Table~\ref{tab:diag} replaces it with agreement on decisions with ammunition left, computed before any evaluation of the new decider. A run-level summary field for held ticks recorded zero in every asynchronous run because it counted only modelled delays; all skipped-tick figures come from the per-episode engine clock. Training of the first DAgger round was interrupted by the host and restarted from the saved collection data before any evaluation. During manuscript preparation, an AI agent asked to typeset the paper returned a version shortened from 7,545 to about 1,100 words while reporting completion, and another dropped most citations from a draft; both outputs were detected by word, page and citation counts and discarded.

\section{Peripheral comparators and host incidents}
\label{app:llm}
Under tick conversion, a 1.7B-parameter open-weight language model (qwen3:1.7b served by Ollama; Table~\ref{tab:env}) choosing every action fired no shots without a notation legend or with eight worked examples and hit nothing with the legend; a commercial language-model agent (gpt-6-sol through the Codex command-line tool 0.156.1 with its default settings, run on 2026-09-23; Table~\ref{tab:env}) at a median of more than five seconds per decision killed 0.43 enemies per game; asking the small model for 16 actions at a time raised accuracy to 8.6\% and kills to 1.93. A slow-fast hybrid in which the small model picks a target about once per second while the teacher's rules aim and fire hit 100.0\% of its shots; its skill comes from the teacher's rules. Per-decision cost differs by orders of magnitude (Figure~\ref{fig:cost}). In earlier diagnostic asynchronous runs the zero-latency teacher hit 92.4\%, 84.8\% and 82.8\% of its shots with 6.3\%, 15.3\% and 15.3\% skipped ticks at low mean processor load; raising the timer resolution did not change the skip rate.

\begin{figure}[H]
\centering
\includegraphics[width=0.82\linewidth]{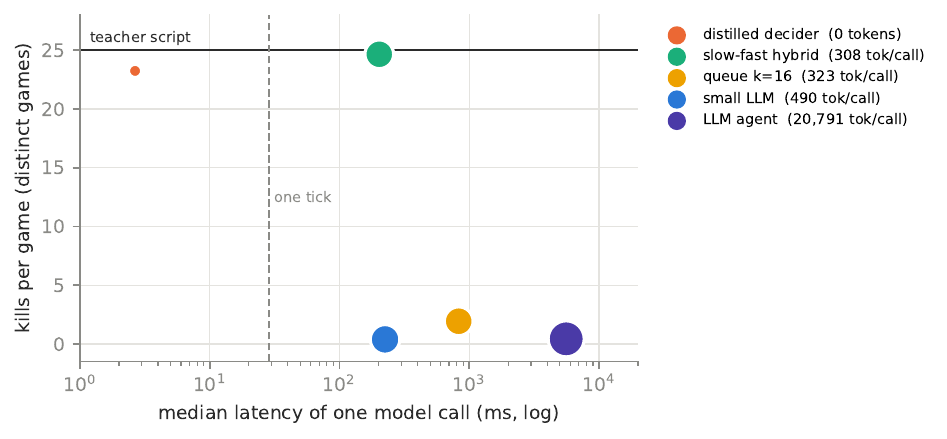}
\caption{Kills per game against median latency of one model call, under tick conversion (peripheral comparison).}
\label{fig:cost}
\end{figure}

\begin{table}[H]
\centering
\caption{Experimental environment. Weight digests are the first 10 hexadecimal digits of SHA-256.}
\label{tab:env}
\small
\begin{tabular}{p{0.26\linewidth}p{0.66\linewidth}}
\toprule
Host 1 & AMD Ryzen 7 7800X3D (8 cores, 16 threads), 31 GB RAM, NVIDIA GeForce RTX 3060, Windows 11 Education 10.0.26200 \\
Host 2 & AMD Ryzen 5 5600G (6 cores, 12 threads), 31 GB RAM, Ubuntu 24.04.4 LTS, Linux 6.11.0-29-generic \\
Game and runtime & ViZDoom 1.3.1 (asynchronous player mode, 35 ticks per second), NumPy 2.5.3; Python 3.14.0 (host 1), 3.12.3 (host 2) \\
Training & PyTorch 2.13.0 CPU (host 1), 2.14.0 CPU (host 2); 2,500 steps, training seed 31 (and 32 for N32) \\
Decider & hashed n-gram encoder to 1,785 features, 265,665 parameters, served over local HTTP; original O a814587e03, after one DAgger round 240805651d, after two rounds N 8ce6b41475, seed-32 replicate 2f4785728c, data-volume control M 2061ae8b; second scene O a164d90514, N 11409b5c60 \\
Open-weight comparator & qwen3:1.7b through Ollama on host 1 (model layer 3d0b790534) \\
Commercial comparator & gpt-6-sol through the Codex command-line tool 0.156.1, default settings, run on 2026-09-23 \\
\bottomrule
\end{tabular}
\end{table}

\section{Trace replay (preregistration v9): a failed fidelity check}
\label{app:v9}
To enforce identical skip schedules by design, preregistration v9 recorded asynchronous runs of the arms N, O and S (seeds 860000--860049, two chunks of 25) together with the tick at which each decision was observed and applied, and replayed these schedules tick by tick in synchronous mode for six policies (teacher, original decider, the deciders after one and two DAgger rounds, the seed-32 replicate and the data-volume control), plus synthetic schedules that skip each tick independently with probability 0 to 0.30 or in bursts. The preregistered first check (B1) replays each recorded arm's own policy on its own schedule and counts episodes whose shots, hits and kills match exactly; below 90\% the replay was not to be read as the asynchronous run.

\begin{table}[H]
\centering
\caption{Fidelity of trace replay (v9, B1). Accuracy and skipped-tick share pooled over the 50 episodes of each arm.}
\label{tab:v9}
\small
\begin{tabular}{@{}llrrrr@{}}
\toprule
Recorded arm & Exact matches & \multicolumn{2}{c}{Hit accuracy (\%)} & \multicolumn{2}{c}{Skipped ticks (\%)} \\
 & & recorded & replayed & recorded & replayed \\ \midrule
N, DAgger decider & 0 of 50 & 87.8 & 99.2 & 8.2 & 8.2 \\
O, original decider & 0 of 50 & 80.9 & 98.8 & 8.6 & 8.7 \\
S, teacher & 0 of 50 & 85.2 & 99.3 & 9.6 & 9.6 \\
\bottomrule
\end{tabular}
\end{table}

Replay reproduced the skipped-tick share of every arm but none of the 150 episodes, and it raised hit accuracy to about 99\% for all three policies. In 124 of the 150 replayed episodes the number of kills exceeds the number of shots counted, which happens in none of the recorded asynchronous episodes. The replay loop presses the previous button again on every tick without a new decision, while the shot counter records shots only on ticks where a new decision is applied; shots fired while a button is held therefore kill enemies without being counted. The replay thus implements the held-action rule of tick conversion rather than asynchronous play, and its own self-test, which compared it with the tick-conversion runner, could not detect this. Under the preregistered rules none of the three hypotheses is supported: the gain of the DAgger decider over the original on the schedules recorded for N is $-$0.25 points [$-$1.01, +0.52] (29 distinct games), the time components do not follow the order of the skipped-tick shares, and the time component at a 30\% synthetic skip rate is [$-$0.08, +1.06] points. Because the fidelity check failed, we do not use these numbers as evidence about asynchronous play. Because the replay's shot counter misses shots fired while a button is held, its near-99\% accuracy is not evidence about held actions either. A separate tick-level audit of the asynchronous engine settled what happens on pending ticks. With a fixed wait of 0, 60 or 120\,ms before each action, each action step spanned 1, 3 or 5 engine ticks. Among decisions that chose no turn right after a turn (163, 40 and 14 decisions), the view angle changed in every one; among those that followed another action (704, 18 and 2 decisions), it changed in none. The engine therefore keeps the previous action on pending ticks, as stated in Section~\ref{sec:setting}. The cause was found later: in asynchronous mode the engine refreshes its clock only when an action is sent, so the recorded application ticks always equalled the observation ticks and the replay acted on fresh states. An emulator that applies each decision one tick after the next observation reproduces the recorded games (Section~\ref{sec:controlled}). A curve study over added delays (v8) was withdrawn before any episode ran, in favour of this design.

\section{Controlled-timing emulator}
\label{app:controlled}
The emulator (\texttt{emu\_s1.py}) runs the engine synchronously and rebuilds asynchronous timing from a schedule of observation ticks. Decision $i$ reads the state at its observation tick $t_i$; its action takes effect from tick $t_{i+1}+L$ and holds until the next action takes effect; before the first action no button is pressed. Metrics are taken where the asynchronous runner takes them. Replaying the recorded actions of the six v9 traces (150 episodes) reproduced kills and the death tick in 0, 136 and 0 episodes for $L=0$, 1 and 2, and we fixed $L=1$ before the preregistered checks.

Two preregistered fidelity checks recorded fresh asynchronous episodes on host 1 and let each policy decide again on its own recorded schedule. v16 (seeds 884000--884024) required all six counts (shots, hits, kills, hits taken, damage, death tick) to match exactly in at least 85\% of episodes per policy; the teacher, O and N reached 20, 20 and 21 of 25, and all five game counts matched in 75 of 75. The mismatches were in the death tick: the recorder does not log the pending interval after its last observation, so the emulator fills it from the recorded gaps. v16b (seeds 885000--885024) required the five game counts to match and the death tick to be within two ticks; the teacher, O and N reached 24, 21 and 24 of 25 against the same 85\% threshold, and all five game counts matched in 25, 23 and 25. Hit accuracy matched in all 150 episodes of the two checks. Both gates failed. In v16b every failing episode had a wait of at least 5 ticks after its last observation (up to 19), and its episodes skipped 19.4--39.4\% of ticks, against 7.9--11.3\% in the v9 schedules. v17 (seeds 886000--886024) recorded that wait and appended it to the schedule as a final gap, so that the next observation cannot come before the recorded end. Its criterion, fixed before recording, required the five game counts to match and the emulated end to fall after the last observation and no later than the recorded end, in at least 85\% of episodes per policy. The teacher, O and N reached 25, 25 and 25 of 25, and also 25 of 25 under the v16b criterion. The v17 episodes skipped 8.1--9.8\% of ticks and no wait exceeded 2 ticks, so the correction itself was not exercised on fresh data; applied after the fact to the v16b episodes, the recorded wait and the window criterion together match all 75 (exploratory).

The main emulation imposed the v9 schedules of S, O and N on every policy (seeds 860000--860049, 29 distinct games) and synthetic schedules with independent skipping at 0, 4, 8, 12, 16 and 20\% and bursty skipping at 10\% (seeds 860000--860099, 56 distinct games). We reran it with the final gap of each recorded schedule imposed (v17); the synthetic schedules have no recorded end and are unchanged, and on the recorded schedules the estimates moved by at most 0.15\pp.

\subsection{Controlled-timing emulator: how asynchronous timing was rebuilt}
\label{app:from-emulator}

Proposition~1 requires S and the decider to face the same skip process, and the studies above matched only the share of skipped ticks. We therefore rebuilt asynchronous play in the synchronous engine, where one schedule can be imposed on every policy. In ViZDoom's asynchronous mode the state and the episode clock are refreshed only when an action is sent, so the ticks at which v9 recorded decisions being applied always equalled the observation ticks, and its replay acted on fresh states (Section~\ref{app:v9}). We rebuilt the timing from the recorded observation ticks alone: the action chosen on the state at one observation takes effect one tick after the next observation, and the previous action holds until then. Replaying recorded actions this way reproduced kills and the death tick in 136 of 150 recorded episodes, and in none with a lag of zero or two ticks. With each policy deciding again on its own recorded schedule, hit accuracy matched in all 150 fresh episodes of two preregistered checks and all five game counts in 148, but the death tick differed in some episodes and both preregistered fidelity gates failed narrowly. A third check on fresh seeds (v17), with the wait after the last observation recorded and imposed, passed (Section~\ref{app:controlled}); the results in Section~\ref{sec:controlled} come from the emulation rerun with that schedule.

\subsection{Live identical-delay study (v18)}
\label{app:live}

v18 ran on host 2 with seeds 887000--887099. Distinct games were identified by the tick-converted teacher on the same seeds (100 reference episodes); all 100 seeds gave distinct games. Seed 887000$+j$ used the observation gaps of episode $j \bmod 50$ of the v9 host-1 recordings of O, repeated from the start if an episode lasted longer. After choosing, each arm waited until (gap $-$ 1) $\times$ 28.571\,ms plus a fixed offset had passed since the observation; the offset, 14\,ms, was chosen in a 40-episode pilot on other seeds as the value with the highest share of realised gaps equal to the target (99.91--99.912\% for all four values tried). Arms ran one at a time in blocks of 25 seeds, in an order fixed before recording, and each session checked the frozen hashes of code, schedules and weights before it started.

The preregistered checks were that the realised gap equals the target gap in at least 90\% of decisions in every waiting arm (99.9\% in each) and that at least 85\% of decider episodes skip within 0.8--1.25 times the share of the teacher on the same seed and session (100\% for O, N and N32). Figure~\ref{fig:live_fid} shows target against realised gaps and Figure~\ref{fig:live_runs} the lengths of skipped runs and the per-episode skipped share.

\begin{figure}[H]
\centering
\includegraphics[width=\linewidth]{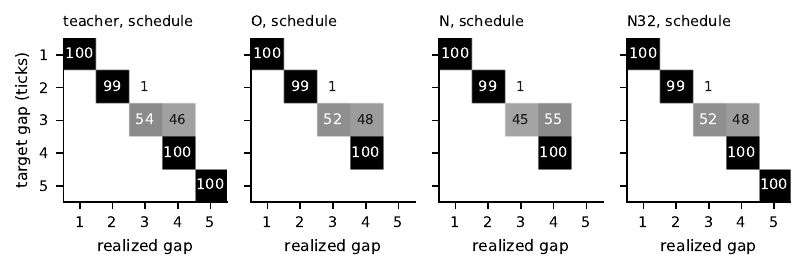}
\caption{Schedule adherence in v18: share of decisions by target and realised gap (ticks), per waiting arm.}
\label{fig:live_fid}
\end{figure}

\begin{figure}[H]
\centering
\includegraphics[width=\linewidth]{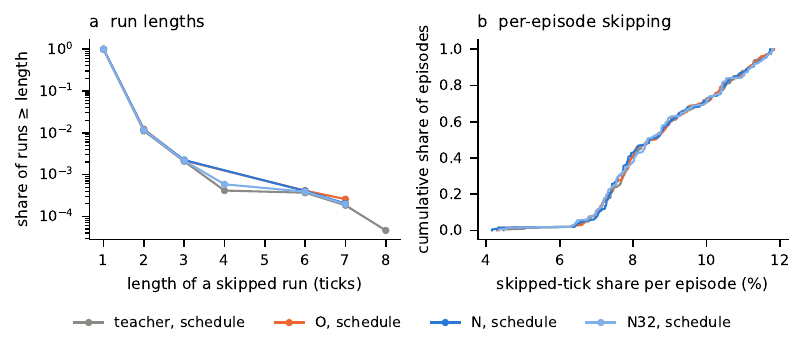}
\caption{Skipping in v18: (a) share of skipped runs at least as long as a given length; (b) cumulative share of episodes by skipped-tick share.}
\label{fig:live_runs}
\end{figure}

Figure~\ref{fig:live_slopes} shows the quality component of O game by game.

\begin{figure}[H]
\centering
\includegraphics[width=\linewidth]{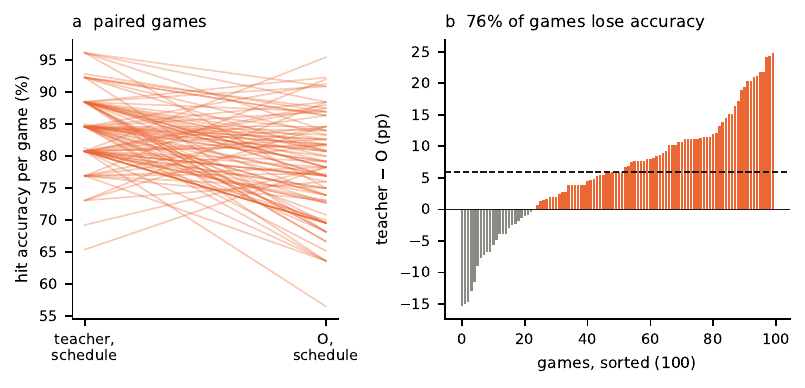}
\caption{Quality component of O per game in v18: (a) teacher and O on the same schedule, one line per game; (b) teacher $-$ O sorted, dashed: mean.}
\label{fig:live_slopes}
\end{figure}

Figure~\ref{fig:live_sessions} compares the two sessions game by game.

\begin{figure}[H]
\centering
\includegraphics[width=\linewidth]{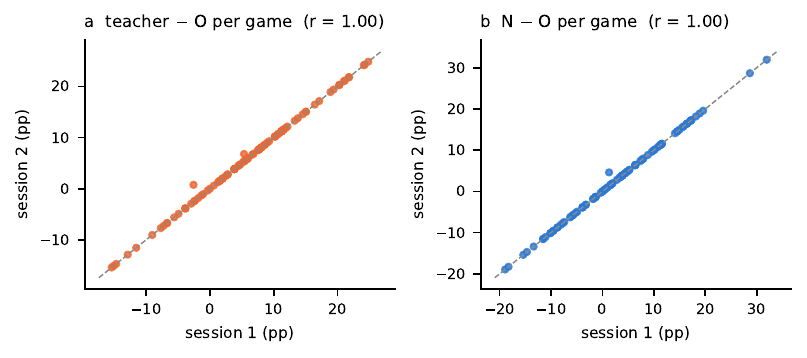}
\caption{Session 1 against session 2 in v18, per game: (a) teacher $-$ O; (b) N $-$ O. Dashed: equality.}
\label{fig:live_sessions}
\end{figure}

\begin{figure}[H]
\centering
\includegraphics[width=\linewidth]{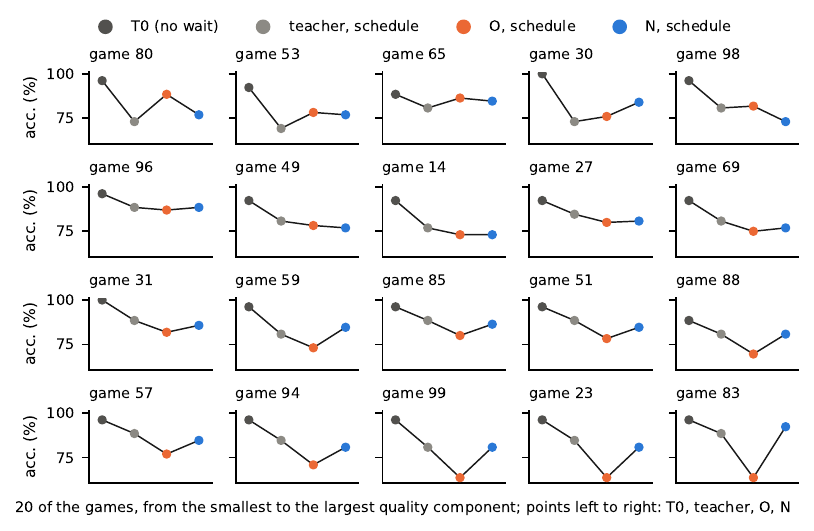}
\caption{T0, teacher, O and N in 20 of the v18 games, from the smallest to the largest quality component of O.}
\label{fig:live_grid}
\end{figure}

Figure~\ref{fig:live_grid} shows the four arms in 20 of the games.

Figure~\ref{fig:live_emu} sets the live estimates beside those of the synchronous emulator.

\begin{figure}[H]
\centering
\includegraphics[width=0.62\linewidth]{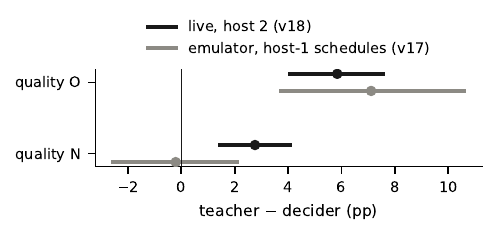}
\caption{Quality components in live play (v18, host 2, 100 games) and in the synchronous emulator on host-1 schedules (v17 rerun, 29 games); different games.}
\label{fig:live_emu}
\end{figure}

Figure~\ref{fig:live_quart} shows accuracy against disagreement within quartiles of the skipped share.

\begin{figure}[H]
\centering
\includegraphics[width=\linewidth]{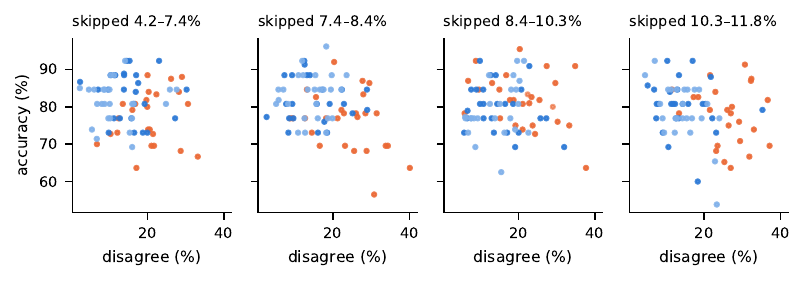}
\caption{Accuracy against disagreement with the teacher per v18 episode, in quartiles of the skipped-tick share; colours as in Figure~\ref{fig:livemech}.}
\label{fig:live_quart}
\end{figure}

\section{Additional views of the data}
\label{app:views}
The figures in this section are generated from the per-episode result files and the offline decision logs by \texttt{make\_figs\_v5.py}; they show the same studies from angles the main text summarises in a single number. A separate data atlas (Supplementary File 2, 34 pages, 46 figures) applies one standard battery to every study of Table~\ref{tab:ledger}, including the failed, withdrawn and exploratory ones: episode distributions, episode-level relation to skipped ticks, secondary metrics, decision time, run order by block, live agreement with the teacher where it was logged, and every interval stored in each analysis file. It is generated by \texttt{make\_study\_battery.py} and \texttt{build\_atlas\_pdf.py} from the result files; a coverage check (\texttt{results\_sprint.py}) confirms that every study has an artifact or a stated reason for each of ten angles.

\begin{figure}[H]
\centering
\includegraphics[width=0.88\linewidth]{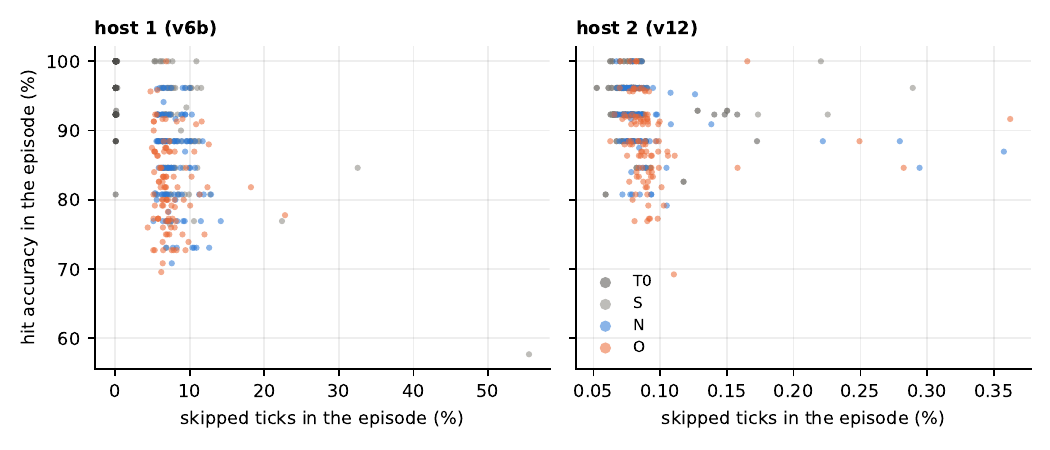}
\caption{Per-episode skipped-tick share and hit accuracy in defend\_the\_center, same 100 seeds on both hosts (v6b, v12). On host 1 the arms that wait for a server skip 5--12\% of ticks in 90\% of episodes (5th--95th percentile), with single episodes up to 56\%; the zero-latency teacher and every arm on host 2 skip below 0.4\%. Note the different horizontal ranges: the hosts differ by two orders of magnitude.}
\label{fig:episodes}
\end{figure}

\begin{figure}[H]
\centering
\includegraphics[width=0.88\linewidth]{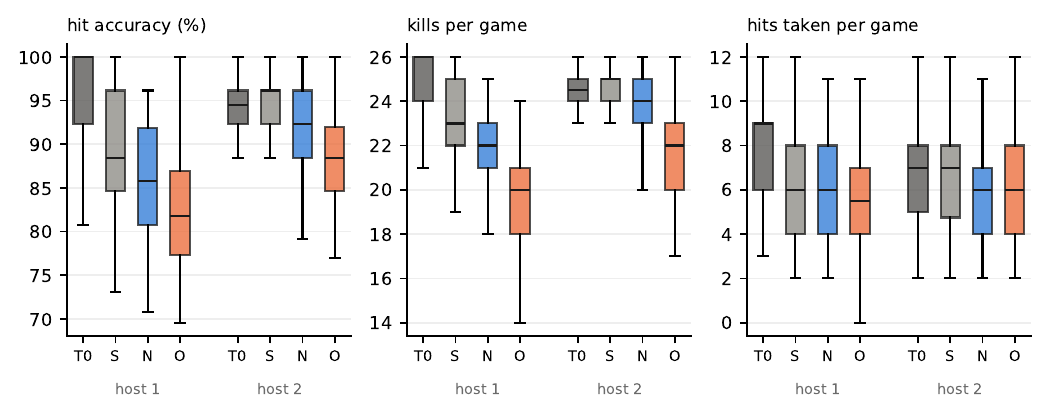}
\caption{Episode-level distributions of three metrics by arm and host (defend\_the\_center, v6b and v12; boxes: quartiles; whiskers: 1.5 interquartile ranges; outliers hidden). On host 2 the latency-matched teacher S is indistinguishable from T0 on all three metrics, while the ordering S $>$ N $>$ O in accuracy and kills persists.}
\label{fig:dists}
\end{figure}

\begin{figure}[H]
\centering
\includegraphics[width=0.88\linewidth]{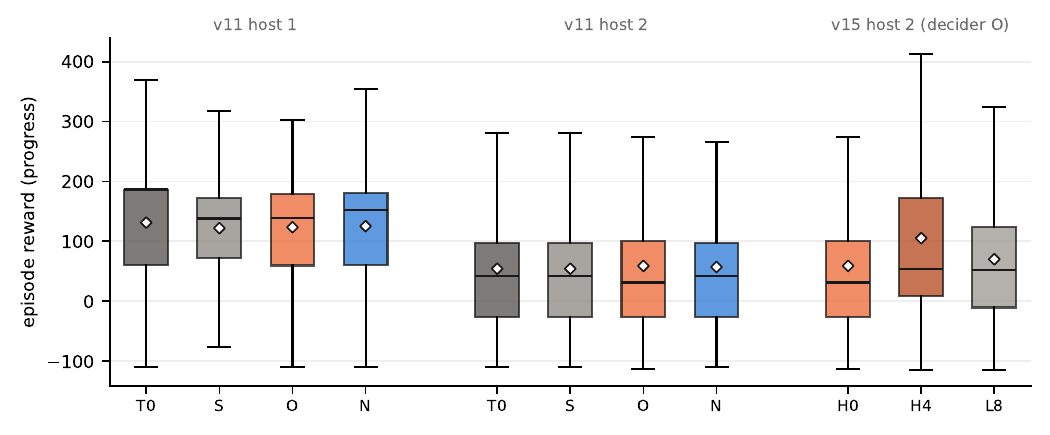}
\caption{Episode rewards in deadly\_corridor (v11 on both hosts, v15 on host 2); diamonds: episode means. Episode rewards vary far more between games than between arms, which is why the game bootstrap intervals of this scene are wide (Table~\ref{tab:hosts}).}
\label{fig:corridor}
\end{figure}

\begin{figure}[H]
\centering
\includegraphics[width=0.88\linewidth]{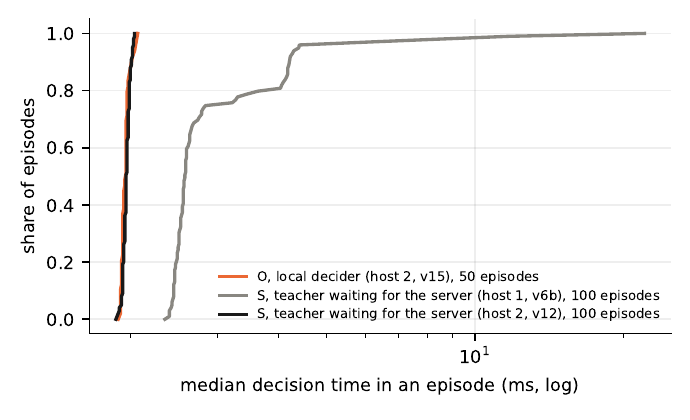}
\caption{Decision time: distribution over episodes of the median decision time (log scale) for the local decider on host 2 and the latency-matched teachers on both hosts, which wait for a real call to the decider's server.}
\label{fig:latency}
\end{figure}

\begin{figure}[H]
\centering
\includegraphics[width=0.88\linewidth]{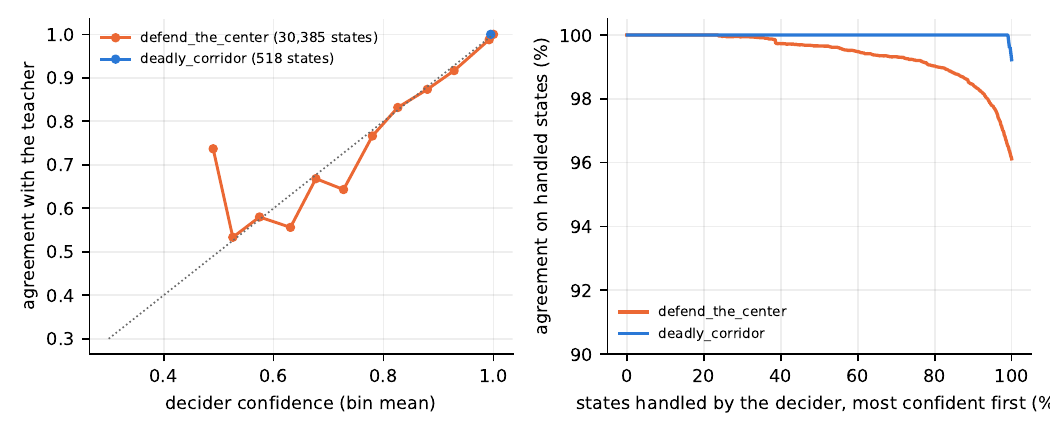}
\caption{Confidence of the original deciders on all held-out teacher states. Left: agreement with the teacher by confidence bin (bins with at least 20 states); the dotted line is perfect calibration. Right: agreement on the states the decider would handle alone when it takes the most confident ones first; the remaining states would be handed to a slower judge (Table~\ref{tab:cascade}).}
\label{fig:calib}
\end{figure}

\begin{table}[H]
\centering
\caption{Confidence thresholds on the validation split (all states). Automated: share of states at or above the threshold; agreement: teacher agreement on those states; handed off: the rest.}
\label{tab:cascade}
\small\renewcommand{\baselinestretch}{1.0}\selectfont
\begin{tabular}{lrrr}
\toprule
Scene and threshold & Automated (\%) & Agreement (\%) & Handed off (\%) \\ \midrule
\multicolumn{4}{l}{\textit{defend\_the\_center (30,385 states)}} \\
none & 100.0 & 96.1 & 0.0 \\
$\ge$ 0.8 & 95.1 & 97.8 & 4.9 \\
$\ge$ 0.9 & 90.8 & 98.4 & 9.2 \\
$\ge$ 0.95 & 84.8 & 98.8 & 15.2 \\
$\ge$ 0.99 & 63.4 & 99.4 & 36.6 \\
\addlinespace
\multicolumn{4}{l}{\textit{deadly\_corridor (518 states)}} \\
none & 100.0 & 99.2 & 0.0 \\
$\ge$ 0.8 & 99.0 & 100.0 & 1.0 \\
$\ge$ 0.9 & 98.7 & 100.0 & 1.4 \\
$\ge$ 0.95 & 97.9 & 100.0 & 2.1 \\
$\ge$ 0.99 & 88.4 & 100.0 & 11.6 \\
\bottomrule
\end{tabular}

\end{table}

\section{Study ledger}
\label{app:ledger}
Table~\ref{tab:ledger} lists every study of the project, including withdrawn, failed and exploratory ones, with the episodes on disk. It is generated by a script (\texttt{study\_ledger.py}) from the preregistration files, the frozen hash lists and the raw result files; no cell is typed by hand. ``Amended'' means that text was appended to the preregistration after it was frozen (results, addenda or a withdrawal note, each marked as such in the file). Episode counts include runs that were repeated after an infrastructure failure. v14 and v14b are withdrawn under the terms of the third-party service they used; their rows keep the episodes on disk and report no results. Their preregistrations committed us to report the results whatever they were, so this withdrawal is a deviation from them.
{\scriptsize\renewcommand{\baselinestretch}{1.0}\selectfont\setlength{\tabcolsep}{3pt}
\begin{longtable}{>{\raggedright\arraybackslash}p{1.0cm}>{\raggedright\arraybackslash}p{4.6cm}c>{\raggedright\arraybackslash}p{1.1cm}r>{\raggedright\arraybackslash}p{1.0cm}>{\raggedright\arraybackslash}p{2.6cm}}
\caption{All studies. Host 1: Windows desktop; host 2: Linux workstation. Scene: defend = defend\_the\_center, corridor = deadly\_corridor. Verdict: the preregistered verdict keys stored in each analysis file.}\label{tab:ledger}\\
\toprule
ID & Study & Host & Scene & Episodes & Frozen text & Verdict \\ \midrule
\endfirsthead
\toprule
ID & Study & Host & Scene & Episodes & Frozen text & Verdict \\ \midrule
\endhead
S1/E1 & first real-time demonstration: paused, tick-converted and asynchronous play & 1 & defend & 320 & yes & see text \\
E2 & competing methods (hybrid, action queue, language models) on one real-time ruler & 1 & defend & 1,365 & amended & see text \\
v3 & asynchronous replication of teacher and decider & 1 & defend & 210 & yes & -- \\
v4 & latency-matched decomposition, three pairs (episodes in v4b: its first three pairs) & 1 & defend & 0 & yes & -- \\
v4b & extension to five pairs; run- and game-level intervals & 1 & defend & 300 & yes & claim of quality component $<$5\,pp allowed: no \\
v5 & on-policy asynchronous DAgger, five triads & 1 & defend & 450 & yes & partial success \\
v6a & calibration of core-pinned concurrent evaluation (failed) & 1 & defend & 90 & amended & -- \\
v6b & replication on 56 distinct games, seed-32 replicate & 1 & defend & 600 & yes & partial success \\
v6b+ & v6b addendum: arms run one at a time after the failed calibration (episodes in v6b) & 1 & defend & 0 & yes & -- \\
v7 & data-volume control & 1 & defend & 200 & yes & not explained by volume \\
v8 & asynchronous latency curve (withdrawn before execution) & 1 & defend & 0 & amended & -- \\
v9 & trace replay of skip schedules (fidelity check failed) & 1 & defend & 5,250 & yes & fidelity check failed; H9-1 supported: no; H9-3 supported: no \\
v10 & sentence-reading teacher: input versus policy share & 1 & defend & 300 & yes & input share small \\
v11 & second scene: decomposition and DAgger & 2 & corridor & 1,460 & yes & inconclusive \\
v11-h1 & second scene on host 1 (addendum) & 1 & corridor & 1,000 & yes & inconclusive \\
v12 & second host, same seeds as v6b & 2 & defend & 500 & yes & partial success \\
v13p & injected-delay pilot, exploratory, no preregistration & 2 & corridor & 175 & -- & -- \\
v14 & withdrawn: third-party service terms; results not reported & 2 & both & 1,500 & yes & not reported \\
v14b & withdrawn: third-party service terms; results not reported & 2 & both & 0 & yes & not reported \\
v15 & action holding for k ticks and stale decisions & 2 & both & 1,800 & yes & negative; positive \\
v16 & controlled-timing emulator: fidelity check and identical-schedule decomposition & 1 & defend & 2,625 & yes & fidelity gate failed; rerun confirmed in v17 \\
v16b & confirmatory fidelity recheck of the emulator & 1 & defend & 75 & yes & fidelity gate failed \\
v17 & emulator fidelity on fresh seeds, censored final interval; identical-schedule rerun & 1 & defend & 75 & yes & fidelity gate passed \\
v18 & live play, one delay schedule imposed on every arm; second DAgger seed; disagreement & 2 & defend & 1,140 & yes & checks passed; all four hypotheses held \\
f1 & conveyor simulator: policy $\times$ injected delay, four concurrent episodes & 2 & conveyor & 600 & yes & see text \\
f1-solo & conveyor simulator: decider one episode at a time, exploratory & 2 & conveyor & 15 & yes & -- \\
corr-att & corridor: learning when to shoot, rounds 1--2 (paused evaluation) & 1 & corridor & 500 & -- & negative \\
corr-att2 & corridor: learning when to shoot, branch-comparison labels & 1 & corridor & 5,000 & -- & negative \\
\bottomrule

\end{longtable}}

\section{Selective automation with a confidence threshold}
\label{app:cascade}
The decider returns a probability for each action. If its largest probability is below a threshold, the decision can be passed to a slower judge. Table~\ref{tab:cascade} gives, for each threshold, the share of held-out teacher states the decider would handle alone and its agreement with the teacher on those states. The measurement is offline (script \texttt{cascade\_coverage.py}, weights a814587e03 and a164d90514); it says how often the decider is confident and right, not how a cascade would play live.

\subsection{Selective automation (exploratory)}
\label{app:from-cascade}

\paragraph{Selective automation (exploratory)} A local decider that returns a confidence can hand only its uncertain cases to a slower judge. On the 30,385 held-out teacher states of defend\_the\_center, O acts alone on 90.8\% of them at a confidence threshold of 0.9, with 98.4\% agreement on those, against 96.1\% on all states; the remaining 9.2\% would go to a larger model or a person (Supplementary Table~\ref{tab:cascade}). These are offline agreements with the teacher, not live scores.

\section{Second environment: a conveyor with a deadline}
\label{app:factory}
LineTempo is a Python conveyor-inspection simulator written before this study for a separate project; we used a frozen version without changing its code. Parts enter at a camera every 0.24\,m on a belt moving at 1\,m/s and reach a diverter 0.4\,m (400\,ms) later; the simulator runs at 50 ticks per second. Each part has a true defect class, severity, size and location. A deterministic perception script, with detection misses, false alarms and blur, turns the state into a sentence such as ``part: family housing | defect scratch | size small | location edge | severity minor | image clear | line: belt normal | \ldots''. The decider receives the sentence over HTTP and chooses PASS, REJECT (with one of four bins), REWORK or ESCALATE; the teacher applies a fixed disposal table to the true state. The factory decider uses the same model family and training code as the Doom decider and was trained on 3,000 teacher-labelled sentences from separate seeds. On 300 held-out sentences (seed 920000) it chose the teacher's operation in 97.7\% of cases, but only 21 of the 26 rejections; this is agreement with a synthetic teacher, not inspection accuracy. A decision that has not arrived when the part reaches the diverter is late and the part passes by default. Cost per part weights an escaped defect 100, a false rejection 20, a line stop 10, an escalation 5 and a rework 1.

\begin{table}[H]
\centering
\caption{Preregistrations, seeds and frozen hashes.}
\label{tab:prereg}
\small
\begin{tabular}{@{}l>{\raggedright\arraybackslash}p{0.42\linewidth}>{\raggedright\arraybackslash}p{0.2\linewidth}>{\raggedright\arraybackslash}p{0.2\linewidth}@{}}
\toprule
ID & Content & Evaluation seeds & Hash \\ \midrule
v3 & zero-latency teacher (T0) and decider, asynchronous replication & 850000--850029 & 299dafd18395 \\
v4 & latency-matched decomposition, three pairs & 850000--850029 & d4c13a3bcfce \\
v4b & extension to five pairs; run-level and game-level intervals & 850000--850029 & a317edf316e7 \\
v5 & on-policy asynchronous distillation, five triads & 850000--850029 & 404982264335 \\
v6a & calibration of core-pinned concurrent evaluation & 850000--850029 & b96753c944f2 \\
v6b & replication with contemporaneous T0 and seed-32 replicate & 860000--860099 & 1ab89caac991 \\
v6b+ & sequential execution after failed calibration & 860000--860099 & d9876210356f \\
v7 & data-volume control & 860000--860099 & 75852f3614ff \\
v9 & trace replay; fidelity check failed (Section~\ref{app:v9}) & 860000--860099 & 3b2f9fcad9d4 \\
v10 & sentence-reading teacher: input versus policy share & 860000--860099 & b8d822e52778 \\
v11 & second scene (deadly\_corridor), host 2; host-1 addendum & 870000--870199 & 59819d9d7802 / 36896cdebded \\
v12 & second host (Linux workstation), defend\_the\_center & 860000--860099 & 01b7e9dd4902 \\
v14 & withdrawn (third-party service terms; results not reported, a deviation) & 860000--860099; 870000--870199 & fcf1420fa783 \\
v14b & withdrawn (third-party service terms; results not reported, a deviation) & 860000--860099; 870000--870199 & 80040d299869 \\
v15 & action holding for $k$ ticks and stale decisions, host 2; resume addendum & 860000--860049; 870000--870199 & 1d15c665c9ce / e85b906894b1 \\
f1 & conveyor simulator, policy $\times$ injected delay, host 2 & 950000--950009 & f194ae8008ed \\
v16 & controlled-timing emulator: fidelity check and identical-schedule decomposition, host 1 & 884000--884024; 860000--860099 & 0139149b996c \\
v16b & confirmatory fidelity recheck of the emulator, host 1 & 885000--885024 & d8abf70af5d6 \\
v17 & emulator fidelity on fresh seeds with the final wait recorded; identical-schedule rerun, host 1 & 886000--886024; 860000--860099 & d70ce2bab648 \\
v18 & live play with one delay schedule imposed on every arm; second DAgger seed; disagreement on own states, host 2 & 887000--887099 & 3a94ac70cf68 \\
\bottomrule
\end{tabular}
\end{table}

The preregistration (hash in Table~\ref{tab:prereg}) crossed the policy (exact teacher; the decider D over live HTTP; an untrained decider; a random policy) with a delay added after each answer arrived, so that delay did not queue later requests: 0, 200, 340, 360, 370, 380, 400 and 600\,ms. Each cell ran 10 episodes of 500 parts on seeds 950000--950009, paired by seed. The teacher was also run with a uniform $\pm$50\,ms jitter on the delay, and both policies were run with the belt waiting for each decision (synchronous mode) at 0, 380 and 600\,ms. Line B adds two part families and a packing buffer; in this configuration its settings otherwise equal line A's and the buffer never filled, so it is close to a repeat of line A. Live runs used four concurrent episodes on host 2.

\begin{table}[H]
\centering
\caption{LineTempo, asynchronous mode: mean cost per part and, in parentheses, late decisions (\%), by injected delay. Teacher: exact teacher; teacher $\pm$50: with $\pm$50\,ms uniform jitter; D: the decider over live HTTP, four concurrent episodes; D, solo: exploratory rerun of D one episode at a time (5 seeds). Other cells: 10 episodes of 500 parts.}
\label{tab:factory}
\footnotesize
\setlength{\tabcolsep}{3.5pt}
\begin{tabular}{@{}rllllll@{}}
\toprule
 & \multicolumn{3}{c}{Line A} & \multicolumn{2}{c}{Line B} & Line A \\
\cmidrule(lr){2-4}\cmidrule(lr){5-6}\cmidrule(l){7-7}
$\delta$ (ms) & teacher & teacher $\pm$50 & D & teacher & D & D, solo \\
\midrule
0 & 0.03 (0.0) & 0.03 (0.0) & 0.88 (0.0) & 0.03 (0.0) & 0.88 (0.0) & -- \\
200 & 0.03 (0.0) & 0.03 (0.0) & 0.88 (0.0) & 0.03 (0.0) & 0.88 (0.0) & -- \\
340 & 0.03 (0.0) & 0.85 (9.6) & 1.37 (7.3) & 0.03 (0.0) & 0.90 (0.2) & 0.92 (0.2) \\
360 & 0.03 (0.0) & 2.06 (30.4) & 1.83 (13.1) & 0.03 (0.0) & 2.99 (31.2) & 0.88 (0.4) \\
370 & 0.03 (0.0) & 3.16 (40.0) & 1.25 (6.0) & 0.03 (0.0) & 1.38 (6.7) & 0.88 (0.3) \\
380 & 0.03 (0.0) & 4.00 (51.0) & 7.52 (100.0) & 0.03 (0.0) & 7.50 (100.0) & -- \\
400 & 7.52 (100.0) & 5.43 (69.9) & 7.52 (100.0) & 7.50 (100.0) & 7.50 (100.0) & -- \\
600 & 7.52 (100.0) & 7.52 (100.0) & 7.52 (100.0) & 7.50 (100.0) & 7.50 (100.0) & -- \\
\midrule
\multicolumn{3}{@{}l}{untrained decider, $\delta=0$} & 5.00 (0.0) & & & \\
\multicolumn{3}{@{}l}{untrained decider, $\delta=600$} & 7.52 (100.0) & & & \\
\multicolumn{3}{@{}l}{random policy} & 8.87 (0.0) & & & \\
\bottomrule
\end{tabular}

\end{table}

Table~\ref{tab:factory} lists every asynchronous cell. The preregistered prediction that the decider would change sharply between 370 and 380\,ms failed: with four concurrent episodes it was late on 7.3\% of parts at 340\,ms (line A), and on up to 82\% in single episodes near the deadline. A probe of the server alone ruled out a transport stall. The exploratory rerun of the decider at 340, 360 and 370\,ms one episode at a time (seeds 950000--950004) gave mean late shares of 0.2--0.4\% (0.2--0.8\% in single episodes) and round-trip p99 of 4.1--4.3\,ms, against 66.7\,ms over the concurrent runs.

In synchronous mode, where the belt waits for each decision, delay barely changed the cost of either policy (decider 0.88, 0.88 and 0.90 per part at 0, 380 and 600\,ms; exact teacher 0.03 throughout), but it cut throughput: the decider line passed 148.7, 76.5 and 59.8 parts per minute and the teacher line 150.0, 78.9 and 61.2, against 249.7 in asynchronous mode. Waiting for each decision thus moves the cost of time from the score into throughput.

\subsection{Deadline edge and host latency in the conveyor simulator}
\label{app:from-factory}

The deadline also included the host. We had predicted the decider's edge between 370 and 380\,ms, where its round trip (median 3.7\,ms) and rounding to the next 20\,ms tick first push answers past 400\,ms. With four episodes running at once on host 2, round trips had a heavy tail (p99 66.7\,ms), and the decider was already late on 7.3\% of parts at 340\,ms. In an exploratory rerun with one episode at a time, mean late shares fell to 0.2--0.4\% (0.2--0.8\% in single episodes) and p99 to 4.1--4.3\,ms. As in Doom, the host sets the time component; with a deadline, a change in the latency tail moves the decider over a cliff rather than along a slope.

\subsection{Simulator checks (exploratory)}
\label{app:factory-checks}

Before the preregistered study, the Python reference of the simulator passed its 58 reference tests. On line B, in a check configuration that differs from the preregistered study (defect rate 0.15 rather than 0.08, 50 parts per episode), we ran one episode on the same seed for each of three policies. The exact teacher, run in process, let none of the 5 defective parts escape. The decider D over live HTTP, with a median decision time of 20\,ms, let 1 escape. A mock slow model whose first answer takes 600\,ms, also over HTTP, was late on every decision and let all 5 escape. These single episodes check the pipeline; they are not estimates. Figure~\ref{fig:env}b is a frame of the decider's episode, replayed from its log in the simulator's viewer.

\begin{figure}[H]
\centering
\includegraphics[width=\linewidth]{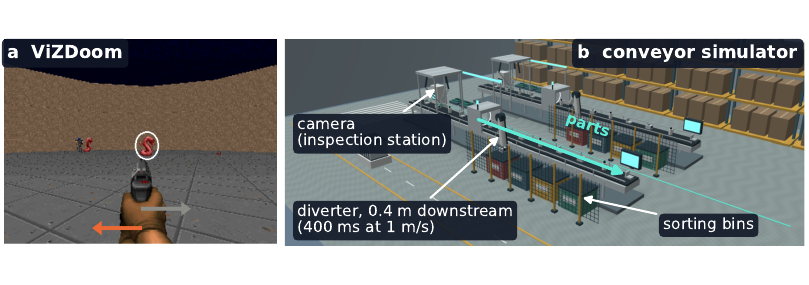}
\caption{(a) ViZDoom from the player's view, with an opposite-turn error of Supplementary Figure~\ref{fig:qual} (gray arrow teacher, orange arrow decider). (b) The conveyor simulator of Section~\ref{sec:factory}, a frame of a logged 50-part run of the decider on line B. A camera inspects each part; a diverter 0.4\,m downstream (400\,ms at 1\,m/s) sorts it.}
\label{fig:env}
\end{figure}

\begin{figure}[H]
\centering
\includegraphics[width=\linewidth]{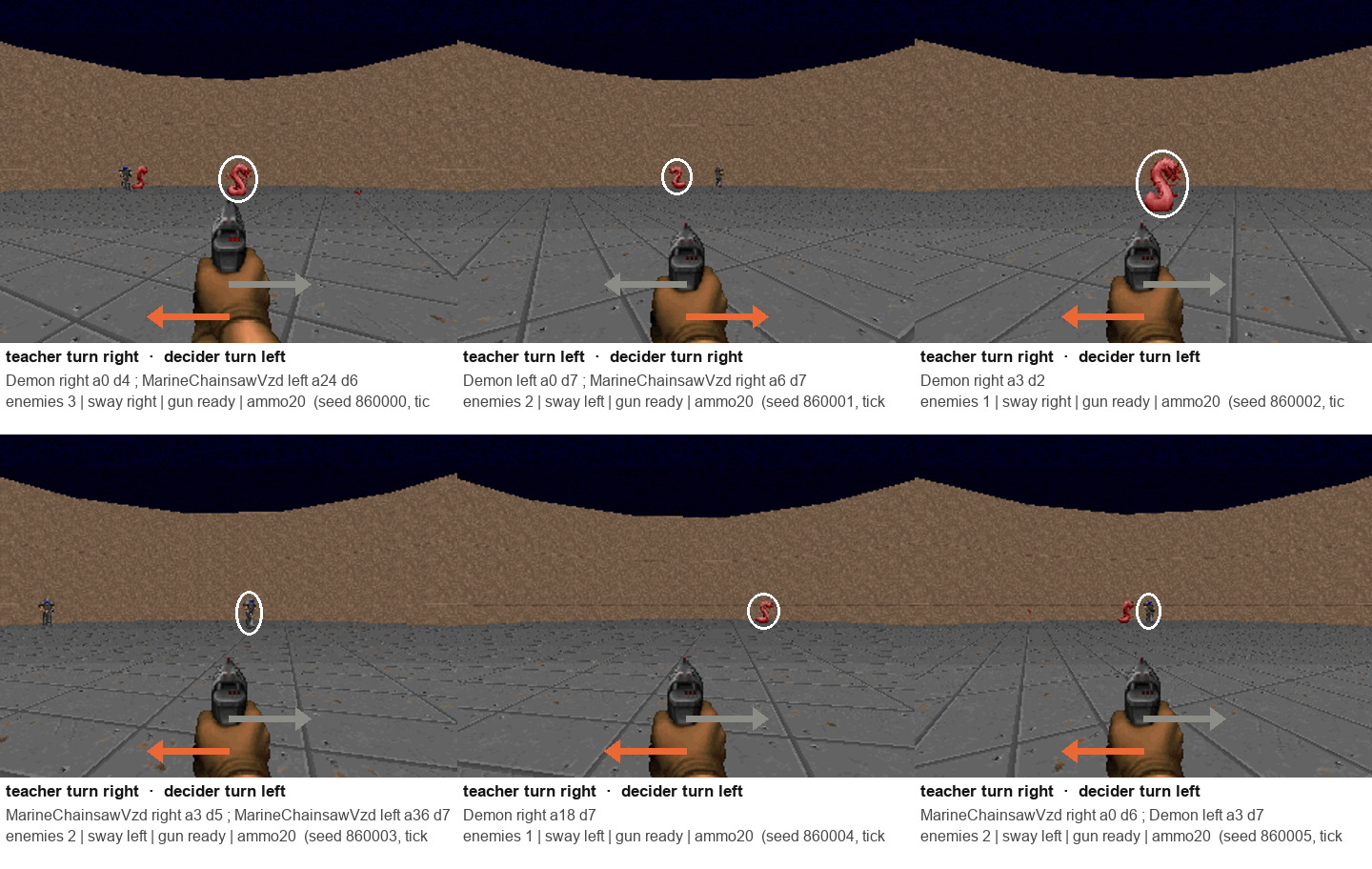}
\caption{First observed opposite-turn errors of the original decider, one per game: the teacher turns toward the nearest enemy and the decider turns away. Key: gray arrow teacher, orange arrow decider, white circle the nearest enemy; in the state sentence, a is the angle bin in degrees and d the distance bin in 100-unit steps. Illustrative, not a measurement: the first opposite-turn decision with ammunition left in each of the first six games from seed 860000, under synchronous replay with 10\% i.i.d.\ skipped ticks. The frame is shown for the reader; the decider itself received only the state sentence printed below it. }
\label{fig:qual}
\end{figure}

\section{Injected-delay pilot}
\label{app:v13}
An exploratory pilot on host 2 (25 seeds of deadly\_corridor, 890000--890024, disjoint from all evaluations; no preregistration) added a fixed delay after every call of S and O to the decider's server. Table~\ref{tab:v13} shows mean reward rising with the delay, for the teacher and the decider alike, as more ticks pass under the previous action. The delay here makes the decision stale; the preregistered study of \emph{holding} a freshly chosen action for $k$ ticks (v15) is reported in Section~\ref{sec:hosts}.

\begin{table}[H]
\centering
\caption{Injected-delay pilot, deadly\_corridor, host 2 (exploratory). T0 has no server call and is the same at every delay.}
\label{tab:v13}
\small\renewcommand{\baselinestretch}{1.0}\selectfont
\begin{tabular}{rrrrrr}
\toprule
Delay (ms) & T0 & S & O & S skipped (\%) & O skipped (\%) \\ \midrule
0 & 61.6 & 61.6 & 63.0 & 1.6 & 1.6 \\
30 & 61.6 & 89.4 & 91.5 & 48.6 & 48.6 \\
60 & 61.6 & 82.4 & 84.5 & 63.1 & 63.1 \\
\bottomrule

\end{tabular}
\end{table}

\section{Reproducing the studies}
\label{app:repro}
Every study is run by one shell script per host that freezes the hashes of its preregistration, runner, analysis script and decider weights before the first seed, runs the arms one at a time in the preregistered order, and calls the analysis script at the end. The evaluators are \texttt{arms\_s1.py} (defend\_the\_center) and \texttt{arms\_corr.py} (deadly\_corridor) in the asynchronous mode; the decider is served by \texttt{doom\_server.py}; injected delay and action holding are wrappers (\texttt{inject\_delay.py}, \texttt{hold\_ticks.py}) that leave the evaluators unchanged. Analyses (\texttt{analyze\_v4b.py} to \texttt{analyze\_v15.py}) run a self-test on synthetic positive and negative controls before reading results. The study ledger, the tables of this supplement and the figures are regenerated from the result files by \texttt{study\_ledger.py}, \texttt{make\_tables\_v021.py} and \texttt{make\_figs\_v4.py} to \texttt{make\_figs\_v6.py}. When a host failed during a run, the remaining jobs were rerun from their first seed under an addendum frozen before resumption (v10, v11 host 1, v15).

\section{Experimental environment}
\label{app:env}
Table~\ref{tab:env} lists the hosts, software and models. The scene-1 studies v3--v10 (defend\_the\_center) ran on host 1; the second scene (deadly\_corridor, v11) and the second-host replication of scene 1 (v12) ran on host 2. The decider and the game engine run on the CPU; the GPU of host 1 served only the open-weight language-model comparator.

\section{Preregistrations}
\label{app:prereg}
Table~\ref{tab:prereg} lists the preregistrations used in this paper with the first 12 hexadecimal digits of the SHA-256 hash of each text as frozen before execution (results were appended to some texts afterwards, in a separate section). Decider weights: original a814587e03, after one DAgger round 240805651d, after two rounds 8ce6b41475, seed-32 replicate 2f4785728c.

\end{document}